\documentclass{article}
\usepackage{kr}
\usepackage{times}
\usepackage{soul}
\usepackage{url}
\usepackage[hidelinks]{hyperref}
\usepackage[utf8]{inputenc}
\usepackage[small]{caption}
\usepackage{graphicx}
\usepackage{amsmath}
\usepackage{amsthm}
\usepackage{thmtools}
\usepackage{thm-restate}
\usepackage{booktabs}
\usepackage{algorithm}
\usepackage{algorithmic}

\usepackage{amssymb}
\usepackage{bbm} 
\usepackage{color}
\usepackage[svgnames]{xcolor} 
\usepackage{comment}
\usepackage{graphics}
\usepackage{graphicx}
\usepackage{mathrsfs} 
\usepackage{stmaryrd}
\usepackage{tikz} 
\usetikzlibrary{shapes.misc,automata,positioning,arrows}
\usepgflibrary{shapes.arrows} 

\newcommand{\ie}{\mbox{i.e.}}
\newcommand{\eg}{\mbox{e.g.}}
\newcommand{\cf}{\mbox{cf.}}

\newcommand{\wrt}{\mbox{w.r.t.}}
\newcommand{\suchthat}{\mbox{s.t.}}
\newcommand{\wolog}{\mbox{w.l.o.g.}}
\newcommand{\etal}{\mbox{et al}}

\newcommand{\df}[1]{\textbf{#1}} 

\newcommand{\tuple}[1]{\langle #1 \rangle}
\newcommand{\defined}{\:\raisebox{1ex}{\scalebox{0.5}{\ensuremath{\mathrm{def}}}}\hskip-1.65ex\raisebox{-0.1ex}{\ensuremath{=}}\:}

\newtheorem{definition}{Definition}
\newtheorem{lemma}{Lemma}
\newtheorem{theorem}{Theorem}
\newtheorem{proposition}{Proposition}
\newtheorem{example}{Example}

\newcommand{\Prp}{\ensuremath{\mathcal{P}}}
\newcommand{\Lang}{\ensuremath{\mathcal{L}}} 

\newcommand{\KB}{\ensuremath{\mathcal{KB}}} 

\newcommand{\U}{\ensuremath{\mathcal{U}}} 
\newcommand{\Val}{\ensuremath{\mathcal{V}}}

\newcommand{\Mod}[1]{\llbracket#1\rrbracket} 
\newcommand{\Cn}{\ensuremath{\text{\it Cn}}} 

\newcommand{\limp}{\rightarrow} 

\newcommand{\sat}{\Vdash}

\newcommand{\entails}{\models}
\newcommand{\nentails}{\not\entails}

\newcommand{\RM}{\mathscr{R}} 
\newcommand{\IM}{\mathscr{I}} 
\newcommand{\BM}{\mathscr{B}} 

\newcommand{\lo}{\mathscr{L}}
\newcommand{\up}{\mathscr{U}}

\newcommand{\rev}{\ast}
\newcommand{\agmrev}{\rev}
\newcommand{\irev}{\ast}
\newcommand{\brev}{\bullet}

\newcommand{\nprev}{\circ}
\newcommand{\opt}{\mathrm{opt}} 

\newcommand{\tpo}{\preceq^{\mathsf{tpo}}}
\newcommand{\intord}{\preceq^{\mathsf{int}}}
\newcommand{\biord}{\preceq^{\mathsf{bi}}}

\newcommand{\CredSet}{\ensuremath{\mathcal{C}}}

\title{Interval Orders, Biorders and Credibility-limited Belief Revision}

\author{%
Richard Booth$^1$\and
Ivan Varzinczak$^{2,3,4}$\\
\affiliations
$^1$Cardiff University, United Kingdom\\
$^2$Université Sorbonne Paris Nord, Inserm, Sorbonne Université, Limics, 93017 Bobigny, France\\
$^3$CAIR, University of Cape Town, South Africa\quad $^4$ISTI-CNR, Pisa, Italy\\
\emails
boothr2@cardiff.ac.uk, \ 
ivan.varzinczak@sorbonne-paris-nord.fr
}

\begin{document}

\maketitle

\begin{abstract}
Rational belief revision is commonly viewed as being based on a preference order between possible worlds, with the resulting new belief set being those sentences true in all the most preferred models of the incoming new information. Usually, such a preference order is taken to be a total preorder. Nevertheless, there are other, more general classes of ordering that can also be employed. In this paper, we explore two such classes that have been studied within the theory of rational choice but have seen limited or no application in belief revision. We begin with \emph{interval orders}, introduced by Fishburn in the '80s, which associate with each possible world a nonnegative `interval' of plausibility. We then move on to \emph{biorders}, studied by Aleskerov, Bouyssou, and Monjardet, which generalise interval orders by allowing the intervals to have negative lengths, a feature that can be used to capture a notion of dissonance or instability. We provide axiomatic characterisations of these two resulting families of belief revision operators, as well as of two further families of interest that lie between interval orders and biorders. We show that while biorder-based revisions satisfy the Success postulate, they do not always yield consistent outputs. By modifying their definition to discard inputs that lead to inconsistency as `incredible', we derive new families of so-called non-prioritised revision that satisfy the Consistency postulate, but not the Success one. These families are linked to credibility-limited revision operators of Hansson~\etal., but for which the set of credible sentences does not satisfy the single-sentence closure condition. We argue that the biorder-based approach is well-suited for scenarios where an agent might initially reject new information, but may accept it when presented with additional explanation.

\end{abstract}

\section{Introduction and Motivation}\label{Introduction}


Belief revision is concerned with the incorporation of any new item of information into an agent's belief set in a principled way. The so-called AGM-approach to belief revision~\cite{AlchourronEtAl1985} does so by formulating a formal account of rationality in this setting, usually characterised by a set of postulates capturing the expected behaviour of revision. Traditionally, two pillars of the AGM approach are the postulates called Success and Consistency, which ensure that the new incoming information is always accepted and the result is a consistent belief set (whenever the incoming information is consistent).

As it turns out, the AGM approach assumes an idealised setting, which is not suitable for modelling the behaviour of realistic rational agents. As a result, the past decades have seen many attempts to bring more realism into it. This is one of the motivations underlying the present paper.

Here, we argue that the requirement for a consistent revision result could be dropped. Realistic agents may face new information that is \emph{destabilising}, represented as sentences that just do not cohere with the agent's current beliefs. As an example, adapted from Fermé~\shortcite{Ferme_thesis}, upon learning ``My five-year-old nephew had lunch with King Charles today'', the agent's beliefs get destabilised and switch to `maintenance mode', where some form of `repair', which can take the form of further clarification, is needed.

Alternatively, if we prize consistency, we may drop the blind acceptance of new incoming information, specifically of problematic sentences. This amounts to giving up the Success postulate and is the stance in non-prioritised revision~(NPR). In the example above, the sentence is just ignored or `put on hold' until further clarification or explanation, if any, might make it become acceptable. For instance, by learning that ``King Charles visited a local school today'', the agent is now keen to accept the first claim above.

In the present paper, we tackle both issues above and their interplay, investigating the consequences but mainly the fruitfulness of dropping either Success or Consistency. Our main tools are a family of orderings on possible worlds, more general than the total preorders usually assumed in the AGM approach. More specifically, we look at \emph{interval orders}~\cite{Fishburn1985}, which have already been utilised in the context of belief change, namely for belief contraction~\cite{Rott2014}. But there is a more general type of ordering that has been looked at in economics, and which can result in an empty choice set, namely \emph{biorders}, as studied most extensively by Aleskerov~\etal.~\shortcite{AleskerovEtAl2007}. We bring these into the landscape of belief change, which is part of the novelty of the present work.


The remainder of the paper is structured as follows:
In Section~\ref{Preliminaries}, we give the formal background for the upcoming sections. In Section~\ref{IntervalBasedRevision}, we present our first contribution by completing the picture initiated by Rott~\shortcite{Rott2014}, namely shifting the focus from total preorders in the definition of belief change operators to interval orders. We then generalise our constructions to the broader class of biorders, showing in particular that intervals of negative length can be used to capture an alternative form of impossibility, which we refer to as dissonance or instability (Section~\ref{BiorderBasedRevision}). One of the consequences of this generalisation is the loss of the Consistency postulate. Instead of throwing in the towel at this stage, we try to `remedy' the situation by exploring two subclasses of biorders satisfying a form of transitivity and that lie between interval orders and biorders. These strengthen the constructions and results for the general biorder case, but do not bring back Consistency. This serves as a motivation for an investigation of non-prioritised revision (NPR) grounded on the operators previously defined (Section~\ref{NPRevision}), in which inputs leading to inconsistency are discarded as `incredible'. Given that, in Section~\ref{CredLimitedRevision}, we show how our two subclasses of~NPR can be formulated as credibility-limited revision operators. Before concluding the paper, we discuss some relevant related work (Section~\ref{RelatedWork}). 
Proofs of all stated results are given in an Appendix. 


\section{Formal Preliminaries and Conventions}\label{Preliminaries}

In this section, we provide the required formal background for the remainder of the paper. 

\subsection{Propositional Notation}\label{Notation}

Let~$\Prp$ be a finite and non-empty set of propositional \emph{atoms}. We use $p,q,\ldots$ as meta-variables for atoms. Propositional sentences are denoted by $\alpha,\beta,\ldots$, and are recursively defined in the usual way. 
With~$\Lang$ we denote the set of all propositional sentences.

With $\U\defined\{0,1\}^{\Prp}$, we denote the set of all propositional \emph{valuations} on~$\Prp$, where~$1$ represents truth and~$0$ falsity. We use $v,u,\ldots$, possibly with primes, to denote valuations. Whenever it eases the presentation, we shall represent valuations as sequences of atoms (\eg, $p$) and barred atoms (\eg, $\bar{p}$), in alphabetical order, with the understanding that the presence of a non-barred atom indicates that the atom is `true' in the valuation, while the presence of a barred atom indicates that the atom is `false' in the valuation.

Satisfaction of~$\alpha\in\Lang$ by~$v\in\U$ is defined in the usual truth-functional way and is denoted by $v\sat\alpha$. The set of \emph{models} of a sentence~$\alpha$ is defined as $\Mod{\alpha}\defined\{v\in\U\mid v\sat\alpha\}$. We say sentences~$\alpha$ and~$\beta$ are \emph{logically equivalent}, denoted $\alpha\equiv\beta$, if $\Mod{\alpha}=\Mod{\beta}$. We say~$\alpha$ is \emph{consistent} if~$\Mod{\alpha}\neq\emptyset$. We call~$\alpha$ a \emph{tautology}, denoted $\entails\alpha$, if~$\Mod{\alpha}=\U$.

The notions above are extended to (possibly infinite) sets~$X\subseteq\Lang$ in the usual way, given $\Mod{X}\defined\bigcap_{\alpha\in X}\Mod{\alpha}$. We say~$X$ (classically) \emph{entails} $\alpha\in\Lang$, denoted $X\entails\alpha$, if $\Mod{X}\subseteq\Mod{\alpha}$. Given~$X\subseteq\Lang$, with $\Cn(X)\defined\{\alpha \mid X\entails\alpha\}$ we denote the set of all \emph{logical consequences of}~$X$.


Finally, in the remainder of the paper, we will often be interested in selecting the models of a sentence or those of a set of sentences according to some well-defined criteria (\eg\ a notion of closeness \wrt\ the models one started off with, such as those of an initial knowledge base). Let~$\Val\subseteq\U$, and let~$\leq\ \subseteq\U\times\U$ be an ordering (not necessarily a strict one). The \emph{optimal elements} of~$\Val$ \wrt~$\leq$ are defined as:
\begin{equation}\label{Def:OptModels}
\opt(\Val,\leq)\defined\{v\in\Val \mid v\leq u \text{ for all }u\in\Val\}
\end{equation}

We shall say that the ordering~$\leq$ is \emph{anchored on}~$\Val$ if $\Val = \opt(\U,\leq)$, i.e., $\Val$ is the set of optimal elements overall, and that~$\leq$ is {\em anchored} if furthermore $\Val \neq \emptyset$.

\subsection{Belief Revision}\label{BeliefChange}

Belief change is concerned with the problem of modifying an agent's beliefs in a principled way. The AGM paradigm for belief change~\cite{AlchourronEtAl1985} has become the gold standard in the area. It assumes an underlying logic with a propositional language, as~$\Lang$, and the above consequence operator~$\Cn(\cdot)$.

A \emph{belief set}~$\KB$ is a \emph{deductively closed} set of sentences, \ie, $\KB=\Cn(\KB)$. 
The \emph{revision} of~$\KB$ with~$\alpha$ corresponds to a consistent resulting belief set~$\KB\rev\alpha$ from which~$\alpha$ follows. 
In the AGM paradigm, the expected behaviour of the revision operator is characterised in terms of the postulates below (we assume \wolog~that the input~$\KB$ is \emph{consistent}):

\begin{description}
\item[P0] $\KB \rev \top = \KB$ \hfill (Tautology)
\item[P1] $\KB\rev\alpha=\Cn(\KB\rev\alpha)$ \hfill (Closure)
\item[P2] $\alpha\in\KB\rev\alpha$ \hfill (Success)
\item[P3] $\KB\rev\alpha\subseteq\Cn(\KB\cup\{\alpha\})$ \hfill (Inclusion)
\item[P4] If $\lnot\alpha\notin\KB$, then $\Cn(\KB\cup\{\alpha\})\subseteq\KB\rev\alpha$ \hfill (Vacuity)
\item[P5] If $\alpha\equiv\beta$, then $\KB\rev\alpha=\KB\rev\beta$ \hfill (Extensionality)
\item[P6] If $\alpha\nentails\bot$, then $\KB\rev\alpha\nentails\bot$ \hfill (Consistency)
\item[P7] $\KB\rev\alpha\cap\KB\rev\beta\subseteq\KB\rev(\alpha\lor\beta)$ \hfill (Disjunctive Overlap)
\item[P8] If $\lnot\alpha\notin\KB\rev\beta$, then $\Cn(\KB\rev\beta\cup\{\alpha\})\subseteq\KB\rev(\alpha\land\beta)$\hfill (Subexpansion)
\end{description}

Let~$\agmrev$ be a standard~AGM revision operator, which can be defined in terms of a total preorder (i.e., a transitive and complete relation) on~$\U$ as follows. To a belief set~$\KB$, we associate a total preorder~$\tpo_{\KB}$ \suchthat~$\tpo_{\KB}$ is anchored on~$\Mod{\KB}$ (see paragraph following~(\ref{Def:OptModels}) above), and the models of the new belief set $\KB\agmrev\alpha$ are selected from~$\Mod{\alpha}$ by choosing the optimal ones \wrt\ to $\tpo_{\KB}$, \ie,
\begin{equation}\label{Def:AGMrevision}
\Mod{\KB\agmrev\alpha}\defined\opt(\Mod{\alpha},\tpo_{\KB})
\end{equation}
where the function~$\opt(\cdot)$ is defined as in~(\ref{Def:OptModels}) for~$\tpo_{\KB}$.

The representation result below~\cite{Grove1988,KatsunoMendelzon1991} establishes the correspondence between the postulates above and the definition of~$\agmrev$. 

\begin{theorem}
The revision operator~$\agmrev$ satisfies Postulates~P1-P8 iff for each~$\KB$ there is a total preorder~$\tpo_{\KB}$ anchored on $\Mod{\KB}$ \suchthat, for each~$\alpha\in\Lang$, $\KB\agmrev\alpha$ is determined from~$\tpo_{\KB}$ as in~(\ref{Def:AGMrevision}).
\end{theorem}

In the remainder of the paper, we explore some variations of the definition in~(\ref{Def:AGMrevision}), informed by different orderings in~$\opt(\cdot)$, and their respective impact on the behaviour of the corresponding belief revision operators. Regarding the notation for these, we adopt the following conventions: 
the operator~$\rev$ is used in contexts in which it is assumed both Success and Consistency hold, $\brev$~in contexts in which Success is assumed but Consistency is not, \eg\ in Section~\ref{BiorderBasedRevision}, and~$\nprev$ in contexts in which Consistency is assumed but not necessarily Success, \eg~in Sections~\ref{NPRevision} and~\ref{CredLimitedRevision}.

\section{Interval Order-based Revision}\label{IntervalBasedRevision}

Interval orders, which, roughly, allow for the comparison of intervals, form a well-known subclass of ordering relations generalising total preorders~\cite{Fishburn1985,Myers1999}. They have proven particularly fruitful in the recent literature on non-monotonic reasoning, with Rott's interval-based belief contraction~\cite{Rott2014} and Booth and Varzinczak's weakening of rational closure for defeasible inference~\cite{BoothVarzinczak2021-AAAI}. In this section, we complete the picture initiated by Rott by exploring the role of interval orders in defining a belief \emph{revision} operator.

\subsection{Interval Orders}

An \emph{interval order} over~$\U$ is a binary relation~$\intord$ satisfying reflexivity, i.e., $v \intord v$ for all $v \in \U$, and the Ferrers condition, \ie, for all $v,u,v',u'\in\U$, we have:\footnote{Here we focus on the \emph{weak} reading of the relation rather than the strict one, which was adopted by Rott~\shortcite{Rott2014} and Booth and Varzinczak~\shortcite{BoothVarzinczak2021-AAAI}.}
\begin{equation}\label{FerrersCondition}
\textrm{If } v\intord u\ \textrm{and } v'\intord u',\
\textrm{then }
v \intord u'\ \textrm{or } v'\intord u
\end{equation}

There are alternative ways to characterise interval orders. A particularly useful one is in terms of interval-based interpretations, as done by Booth and Varzinczak~\shortcite{BoothVarzinczak2021-AAAI} in a preferential semantics for defeasible conditionals.

\begin{definition}[Interval-Based Interpretation]\label{Def:IntervalBasedInterpretation}
An \df{interval-based interpretation} is a pair $\IM\defined\tuple{\lo,\up}$, where~$\lo$ and~$\up$ are functions from~$\U$ to $\mathbb{N}$ \suchthat\ for all $v\in\U$, $\lo(v)\leq\up(v)$. 
Given $\IM=\tuple{\lo,\up}$ and~$v\in\U$, $\lo(v)$ is the \df{lower rank of}~$v$ \df{in}~$\IM$, and $\up(v)$ is the \df{upper rank of}~$v$ \df{in}~$\IM$. Hence, the pair~$(\lo(v),\up(v))$ is the \df{interval of}~$v$ \df{in}~$\IM$. 
\end{definition}

We point out that this definition differs slightly from the one by Booth and Varzinczak~\shortcite{BoothVarzinczak2021-AAAI} in that the co-domain of both $\lo$ and~$\up$ is taken to be~$\mathbb{N}$ rather than $\mathbb{N}\cup\{\infty\}$. There, the infinite rank~$\infty$ is interpreted as representing \emph{impossibility}, but, as we shall soon see, our framework will be extendable to express different kinds of incoherence naturally, without the use of infinite ranks.


\begin{example}\label{Ex:IntervalInterpretation}
Let~$\Prp=\{p,q\}$, and let~$\lo$ and~$\up$ be given by~$\lo(pq)=0$, $\lo(p\bar{q})=2$, $\lo(\bar{p}q)=1$, and $\lo(\bar{p}\bar{q})=0$; $\up(pq)=1$, $\up(p\bar{q})=3$, $\up(\bar{p}q)=1$, and $\up(\bar{p}\bar{q})=3$. Then~$\IM=\tuple{\lo,\up}$ is an example of an interval-based interpretation for~$\Prp$.
\end{example}

Figure~\ref{Fig:IntervalInterpretation} provides an easier to visualise representation of the interval-based interpretation in Example~\ref{Ex:IntervalInterpretation}.

\begin{figure}[ht]
\[
\begin{array}{|c|c|c|c|c|}
\hline 
    & 0 & 1 & 2 & 3 \\ \hline
\lo & pq, \bar{p}\bar{q} & \bar{p}q & p\bar{q} & \\ \hline
\up & & pq, \bar{p}q &  & p\bar{q}, \bar{p}\bar{q}\\ \hline
\end{array}
\]
\caption{An interval-based interpretation for~$\Prp=\{p,q\}$.}
\label{Fig:IntervalInterpretation}
\end{figure}

Given an interval-based interpretation~$\IM=\tuple{\lo,\up}$, the relation $\intord_{\IM}$ on~$\U$, defined by $v\intord_{\IM}v'$ if and only if~$\lo(v)\leq\up(v')$, for every~$v,v'\in\U$, is an interval order on~$\U$. In terms of the intervals, that means we have $v\intord_{\IM}v'$ if and only if either $v$'s interval overlaps with $v'$'s one or is strictly `below' it. For~$\IM$ as in Example~\ref{Ex:IntervalInterpretation}, we have, for instance, $(pq,pq),(\bar{p}\bar{q},p\bar{q})\in\intord_{\IM}$, but $(p\bar{q},pq)\notin\intord_{\IM}$. 
Furthermore, it can be checked that for every interval order~$\intord$ over~$\U$, there is an interval-based interpretation~$\IM$ such that~$\intord\ =\ \intord_{\IM}$.




\subsection{IOB Revision}\label{IOB-Revision}


An interval order yields a (successful) belief revision operator~$\irev$ in the same way as it is traditionally done for total preorders. To each belief set~$\KB$, we can associate an interval order~$\intord_{\KB}$ such that the optimal models are the ones in~$\Mod{\KB}$, and then the models of the resulting new belief set~$\KB\irev\alpha$ are selected from~$\Mod{\alpha}$ by picking the optimal ones according to~$\intord_{\KB}$:
\begin{equation}\label{Def-IO-Revision}
\Mod{\KB\irev\alpha}\defined\opt(\Mod{\alpha},\intord_{\KB})
\end{equation}
(We remind the reader that the function~$\opt(\cdot)$ is defined as in~(\ref{Def:OptModels}) above, with~$\intord_{\KB}$ as the informing order here.)

\begin{definition}[Interval Order-Based Revision]
$\irev$ is an \df{interval order-based (IOB) revision operator} if, to each $\KB$ we can associate an interval order~$\intord_{\KB}$ anchored on $\Mod{\KB}$ \suchthat, for each~$\alpha$, $\KB\irev\alpha$ is determined from~$\intord_{\KB}$ as in~(\ref{Def-IO-Revision}) above.
\end{definition}

Not surprisingly, the definition of revision in terms of interval orders comes with immediate consequences regarding the postulates the corresponding operator satisfies. Indeed, with the optimal models now informed by an underlying interval order, the revision operator from~(\ref{Def-IO-Revision}) above fails both Subexpansion~(P8) and Vacuity~(P4) in general. This is witnessed by the following example. 

\begin{example}\label{Ex:FailureVacuitySubexpansion}
Let $\KB=\Cn(\{p\limp q\})$ and let~$\intord_{\KB}=\intord_{\IM}$, with~$\intord_{\IM}$ as in Example~\ref{Ex:IntervalInterpretation}. It is easy to verify that~$\intord_{\KB}$ is anchored on~$\Mod{\KB}$, as $\Mod{\KB}=\opt(\Mod{\KB},\intord_{\KB})$. For the failure of~(P4), assume~$\alpha=\lnot q$. Then we have $q\notin\KB$ and~$\Mod{\KB\irev\lnot q}=\opt(\Mod{\lnot q},\intord_{\KB})=\{p\bar{q},\bar{p}\bar{q}\}$. But we have $\lnot p\in\Cn(\KB\cup\{\lnot q\})$. For the failure of~(P8), let~$\alpha=\lnot q$ and~$\beta=\top$. We have $q\notin\KB\irev\top$ and~$\lnot p\in\Cn(\KB\irev\top\cup\{\lnot q\})$, but $\lnot p\notin\KB\irev(\lnot q\land\top)$.
\end{example}

We consider the following Disjunctive Rationality~(DR) postulate, which is satisfied by the AGM revision operator based on total preorders~$\agmrev$ above:

\begin{description}
\item[P9] $\KB\rev(\alpha\lor\beta)\subseteq\KB\rev\alpha\cup\KB\rev\beta$ \hfill (DR)
\end{description}

This gives us a representation result for IOB revision operators (whose proof can actually be straightforwardly based on the proof characterising the more general BOB revision operators in the next section):

\begin{restatable}{theorem}
{RepResultIOB}
\label{Thm:RepResultIOB}
An operator~$\irev$ is an IOB revision operator iff it satisfies Tautology, Closure, Success, Inclusion, Extensionality, Consistency, Disjunctive Overlap, and Disjunctive Rationality.
\end{restatable}



Even though Vacuity does not hold in general, we note that, in the case one wants it to be enforced, it is sufficient to impose the following condition on~$\intord_{\KB}$ (where~$\prec^{\mathsf{int}}_{\KB}$ below is the \emph{strict} version of~$\preceq_{\KB}$):
\begin{equation}\label{ConditionForVacuity}
\text{If }v\in\Mod{\KB}\text{ and }v'\notin\Mod{\KB}, \text{ then }v\prec^{\mathsf{int}}_{\KB}v'
\end{equation}

We end this section by noting that another, equivalent, way to define~$\Mod{\KB \irev \alpha}$ is directly in terms of an interval-based interpretation (see Definition~\ref{Def:IntervalBasedInterpretation}). Let $\IM=\tuple{\lo,\up}$, and, for every~$\gamma\in\Lang$, let~$\lo(\gamma)\defined\min\{\lo(v)\mid v\in\Mod{\gamma}\}$ and~$\up(\gamma)\defined\min\{\up(v)\mid v\in\Mod{\gamma}\}$. Then:
\begin{equation}
\Mod{\KB \irev \alpha}\defined\{v\in \Mod{\alpha} \mid \lo(v) \leq \up(\alpha)\}
\end{equation}

\section{Biorder-based Revision}\label{BiorderBasedRevision}

We now switch to a class of orderings that have been studied within the theory of rational choice~\cite{AleskerovEtAl2007}, namely biorders, and explore their fruitfulness in belief revision.

\subsection{Biorders}

A \emph{biorder} over~$\U$ is a binary relation~$\biord$ that satisfies the Ferrers condition in~(\ref{FerrersCondition}) above, but that is not necessarily reflexive. Hence, they are a generalisation of the interval orders we have considered in the previous section, in which there are allowed to be valuations~$z\in\U$ such that $z\not\biord z$.

\begin{definition}[Dissonant valuation]\label{Def:DissonantValuation}
Given a biorder~$\biord$ over~$\U$ and~$z\in\U$, we call~$z$ \df{dissonant} (\wrt~$\biord$) if~$z\not\biord z$.
\end{definition}
We will usually reserve~$z$ (possibly with subscripts) to denote dissonant valuations.

There are various ways to characterise biorders~\cite{AleskerovEtAl2007}. A particularly useful, and more visual, one is in terms of two ranking functions, the lower and the upper ones, as in the interval-order case. The main difference is that here the intervals assigned to valuations can be of `negative' length.

More formally, biorders can be characterised via a generalisation of interval-based interpretations:

\begin{definition}[Biorder-Based Interpretation]\label{Def:BiorderBasedInterpretation}
A \df{biorder-based interpretation} is a pair $\BM\defined\tuple{\lo,\up}$, where~$\lo$ and~$\up$ are functions from~$\U$ to $\mathbb{N}$.  Given $\BM=\tuple{\lo,\up}$ and~$v\in\U$, $\lo(v)$ is the \df{lower rank of}~$v$ \df{in}~$\BM$, and $\up(v)$ is the \df{upper rank of}~$v$ \df{in}~$\BM$. Hence, the pair~$(\lo(v),\up(v))$ is the \df{interval of}~$v$ \df{in}~$\BM$. 
\end{definition}

\begin{example}\label{Ex:BiorderInterpretation}
Let~$\Prp=\{p,q\}$, and let~$\lo$ and~$\up$ be such that~$\lo(pq)=2$, $\lo(p\bar{q})=3$, $\lo(\bar{p}q)=4$, and $\lo(\bar{p}\bar{q})=0$; $\up(pq)=3$, $\up(p\bar{q})=1$, $\up(\bar{p}q)=0$, and $\up(\bar{p}\bar{q})=4$. Then~$\BM=\tuple{\lo,\up}$ is an example of a biorder-based interpretation for~$\Prp$.
\end{example}

Figure~\ref{Fig:BiorderInterpretation} provides an easier to visualise representation of the biorder-based interpretation in Example~\ref{Ex:BiorderInterpretation}.

\begin{figure}[ht]
\[
\begin{array}{|c|c|c|c|c|c|}
\hline 
    & 0 & 1 & 2 & 3 & 4\\ \hline
\lo & \bar{p}\bar{q} &  & pq & p\bar{q} & \bar{p}q \\ \hline
\up & \bar{p}q & p\bar{q} & & pq & \bar{p}\bar{q} \\ \hline
\end{array}
\]
\caption{A biorder-based interpretation for~$\Prp=\{p,q\}$.}
\label{Fig:BiorderInterpretation}
\end{figure}

If by the \emph{length} of~$v$'s interval in~$\BM=\tuple{\lo,\up}$ we understand the result of~$\up(v)-\lo(v)$, then~$v$ having a negative interval corresponds to the case where~$\up(v)<\lo(v)$. In the biorder-based interpretation in Example~\ref{Ex:BiorderInterpretation}, $pq$'s interval has length~$1$, $p\bar{q}$'s interval has length~$-2$, $\bar{p}q$'s interval has length~$-4$, and~$\bar{p}\bar{q}$'s interval has length~$4$.

The intuition behind the assignment of a negative interval to a valuation is that it captures the fact the valuation in question bears some form of `impossibility'. Here, this stands as a simple alternative to the adoption of infinite ranks, as in the preferential approaches to defeasible reasoning.

As with interval-based interpretations, each biorder-based interpretation $\BM$ yields a relation $\biord_\BM$ over $\U$ given by $v \biord_\BM v'$ if and only if $\lo(v) \leq \up(v')$. Clearly, the dissonant valuations (see Definition~\ref{Def:DissonantValuation}) for~$\biord_\BM$ are precisely the valuations that get assigned negative intervals in~$\BM$. The following result is essentially shown by Aleskerov \etal.~\shortcite{AleskerovEtAl2007}.

\begin{restatable}{proposition}{RepresentationOfBiorders}
\label{Prop:RepresentationOfBiorders}
Let $\biord \subseteq \U \times \U$. Then $\biord$ is a biorder over~$\U$ iff there is a biorder-based interpretation~$\BM$ \suchthat~$\biord = \biord_\BM$.
\end{restatable}

Given~$\BM$ from Example~\ref{Ex:BiorderInterpretation} above, we have~$\biord_{\BM}=\{(\bar{p}\bar{q},\bar{p}\bar{q}),(\bar{p}\bar{q},\bar{p}q),(\bar{p}\bar{q},p\bar{q}),(\bar{p}\bar{q},pq),(pq,pq),(pq,\bar{p}\bar{q}),$ $(p\bar{q},pq),(p\bar{q},\bar{p}\bar{q}),(\bar{p}q,\bar{p}\bar{q})\}$.

We observe that the mapping $\BM \mapsto \biord_\BM$ is not necessarily one-to-one, as several distinct pairs of ranking functions may yield the same ordering. 
However, the proof of Proposition~\ref{Prop:RepresentationOfBiorders} shows that every biorder can be represented by a {\em compressed} biorder-based interpretation.

\begin{definition}[Compressed Biorder-Based Interpretation]
\label{Def:CompressedBiorderBasedInterpretation}   
A biorder-based interpretation $\BM = \tuple{\lo,\up}$ is \df{compressed} if, {\em (a)} for all $1 \leq i < j$, $\lo^{-1}(j) \neq \emptyset$ implies $\lo^{-1}(i) \neq \emptyset$, {\em (b)} for all $0 \leq i < j$, $\up^{-1}(j) \neq \emptyset$ implies $\up^{-1}(i) \neq \emptyset$, and {\em (c)} $0 \leq \max\lo - \max\up \leq 1$, where $\max\lo = \max\{i \mid \lo^{-1}(i) \neq \emptyset\}$ (and similarly for $\max\up$).
\end{definition}
In words, $\BM$ is compressed if there are no `gaps' in either the lower or the upper rankings (except, possibly, $\lo^{-1}(0)$, which will be empty exactly when $\biord$ has no optimal elements overall), and the largest lower rank is either equal to the largest upper rank or is bigger by one. For example Figure \ref{Fig:CompressedBiorderInterpretation} shows a compressed biorder-based interpretation that gives the same biorder as the one in Figure \ref{Fig:BiorderInterpretation}.

\begin{figure}[ht]
\[
\begin{array}{|c|c|c|c|}
\hline 
    & 0 & 1 & 2 \\ \hline
\lo & \bar{p}\bar{q} &  pq, p\bar{q} & \bar{p}q \\ \hline
\up & \bar{p}q, p\bar{q} &  pq & \bar{p}\bar{q} \\ \hline
\end{array}
\]
\caption{A compressed biorder-based interpretation.}
\label{Fig:CompressedBiorderInterpretation}
\end{figure}

For convenience, we assume that every biorder-based interpretation~$\BM=\tuple{\lo,\up}$ is \emph{normal} in that there is at least one valuation~$v$ such that~$\lo(v)=0$. This ensures that the set of worlds that are optimal overall in~$\U$ is non-empty.

\subsection{BOB Revision}
\label{Sect:BOB_revision}

As we have done for interval orders, we can define a belief revision operator yielded by a biorder. To each (consistent) belief set~$\KB$, we therefore associate a biorder~$\biord_{\KB}$ such that the optimal models overall are the ones in~$\Mod{\KB}$. The models of the resulting new belief set~$\KB\brev\alpha$ are selected from~$\Mod{\alpha}$ by choosing the optimal ones according to~$\biord_{\KB}$:
\begin{equation}\label{Def-BO-Revision}
\Mod{\KB\brev\alpha}\defined\opt(\Mod{\alpha},\biord_{\KB})
\end{equation}

\begin{definition}[Biorder-Based Revision]
$\bullet$ is a biorder-based (BOB) revision operator if, to each $\KB$ we can associate a biorder~$\biord_{\KB}$ anchored on $\Mod{\KB}$ \suchthat, for each~$\alpha$, $\KB\brev\alpha$ is determined from~$\biord_{\KB}$ as in~(\ref{Def-BO-Revision}) above.
\end{definition}

Just as for IOB revision (\cf\ end of Section~\ref{IOB-Revision}), we can define $\Mod{\KB \brev \alpha}$ directly in terms of~$\lo$ and~$\up$ as follows:

\begin{equation}
\Mod{\KB \brev \alpha}\defined\{v\in \Mod{\alpha} \mid \lo(v) \leq \up(\alpha)\}
\end{equation}
where~$\lo(\gamma)=\min\{\lo(v)\mid v\in\Mod{\gamma}\}$ and~$\up(\gamma)=\min\{\up(v)\mid v\in\Mod{\gamma}\}$.
Thus, we can see that $\brev$ will not satisfy Consistency in general. Indeed, if $\up(\alpha)<\lo(\alpha)$, then $\Mod{\KB\brev\alpha}$ will be empty, even if $\Mod{\alpha}$ is non-empty. For an example of such a situation, let~$\KB=\Cn(\{\lnot p\land\lnot q\})$, let~$\alpha=p$, and let~$\biord_{\KB}=\biord_{\BM}$, where~$\BM$ is as in Example~\ref{Ex:BiorderInterpretation}.

\begin{definition}[Destabilising Sentences]\label{Def:DestabilisingSentence}
Given a BOB revision operator $\brev$ and a belief set~$\KB$, we say that~$\alpha$ is \df{destabilising} for~$\brev$ and~$\KB$ if~$\KB\brev\alpha$ is inconsistent.
\end{definition}

A rational agent may have sentences that are destabilising for it.
Intuitively, a destabilising sentence is an incoming piece of information that would make an agent's beliefs `collapse' if accepted. It turns out that one can have~$\alpha$ as destabilising, while~$\alpha\land\beta$, for some~$\beta$, is not. For instance, for~$\KB$ as above, we have that~$p$ is destabilising while~$p\land q$ is not. This corresponds to situations in which some extra information, or context, as called by~Schwind~\shortcite{Schwind2025} (see also Section~\ref{RelatedWork}), makes the acceptance of the sentence compatible with the agent's current beliefs. We shall give a special name to those destabilising sentences that cannot be rendered stable even with the addition of more information.

\begin{definition}[Irreconcilable Sentences]\label{Def:IrreconcilableSentence}
Given a BOB revision operator $\brev$ and a belief set~$\KB$, we say that~$\alpha$ is \df{irreconcilable} for~$\brev$ and~$\KB$ if~$\KB\brev\beta$ is inconsistent for all $\beta \in \Lang$ such that $\beta \entails \alpha$.
\end{definition}

For an example of an irreconcilable sentence, let~$\KB=\Cn(\{\lnot p\land\lnot q\})$ and let~$\biord_{\KB}=\biord_{\BM}$, where~$\BM$ is as in Example~\ref{Ex:BiorderInterpretation}. Then~$\neg p \wedge q$ is irreconcilable. 

We have the following representation result for BOB revision operators, which shows that BOB revision is characterised by the same list of postulates as IOB revision, just without Consistency:

\begin{restatable}
{theorem}
{RepResultBOB}
\label{Thm:RepResultBOB}
An~operator $\brev$ is a BOB revision operator iff it satisfies Tautology, Closure, Success, Inclusion, Extensionality, Disjunctive Overlap, and Disjunctive Rationality. 
\end{restatable}

For the completeness part of the proof of this result, we construct,
from a given operator~$\brev$ satisfying the stated postulates, and for each~$\KB,$ a biorder~$\biord_{\KB,\brev}$ via a compressed biorder-based interpretation $\BM^\brev_\KB = \tuple{\lo^\brev_{\KB},\up^\brev_{\KB}}$ as follows. We inductively build two sequences of subsets of~$\U$, an increasing and a decreasing one:
\[
U_0 \subseteq U_1 \subseteq \cdots \subseteq U_n \quad\text{ and }\quad V_0 \supseteq V_1 \supseteq \cdots \supseteq V_n
\]
    Initially $U_{-1} = \emptyset$ and $V_{-1} = \U$. And then, for each $i \geq -1$:\footnote{Here whenever a valuation appears within the scope of a propositional connective, we are taking it to stand for any sentence which has that valuation as its only model.}
    \[
    U_{i+1}\ \defined\ U_i \cup \Mod{\KB \brev \bigvee V_i}
    \]\[
    V_{i+1}\ \defined\ \bigcup \{Y \subseteq \U \mid  \Mod{\KB \brev \bigvee Y} \nsubseteq U_{i+1}\} 
    \]
    Let $n$ be minimal \suchthat~$n\geq 1$ and $U_{n+1} = U_n$. (Note~$n$ is guaranteed to exist since $\U$ is finite.)
    If, after the procedure terminates, we have $U_n \neq \U$, then we also add~$\U$ to the end of the list to make sure every member of~$\U$ appears in at least one $U_i$ and thus that $\lo^\brev_\KB$ defined below assigns a value to each valuation.  Finally, from these two sequences we read off $\lo^\brev_\KB$ and $\up^\brev_\KB$:
\[
\lo^\brev_\KB(u) \defined \min\{i \mid u \in U_i\}, \ 
\up^\brev_\KB(u) \defined \min\{i \mid u \not\in V_i\}
\]
It can then be shown that $\up^\brev_\KB$ also assigns a value to each valuation, that $\biord_{\KB,\brev}$ is anchored on $\Mod{\KB}$ and that $\opt(\Mod{\alpha}, \biord_{\KB,\brev}) = \Mod{\KB \brev \alpha}$ for each $\alpha$.

\begin{definition}[Canonical Biorder-based Interpretation]
Given a BOB revision operator $\brev$ and belief set $\KB$, the biorder-based interpretation $\BM^\brev_\KB$ defined above is the \df{canonical biorder-based interpretation associated to~$\brev$ and~$\KB$}.
\end{definition}

We observe that two different biorders may yield the same revision behaviour. In particular, for every dissonant valuation~$z$, the value of~$\lo(z)$ is irrelevant, \ie, we can increase~$\lo(z)$ without changing the resulting revision behaviour. This has essentially been pointed out by Aleskerov~\etal.~\shortcite{AleskerovEtAl2007} as well. To put it another way, if, for a given biorder-based interpretation~$\BM$, we have~$z\not\biord_{\BM}z$, then, for all $v \in \U$, the relation given by~$\biord_{\BM}\setminus\{(z,v)\}$, which is a biorder, will yield the same revision behaviour. 
As a result of this, every BOB revision operator~$\brev$ can be generated by a function that assigns to each~$\KB$ a biorder satisfying the following additional property (which does not play any role in the soundness proof of Theorem~\ref{Thm:RepResultBOB} above):
\begin{equation}\label{Property:BOBrevision}
\text{If } v\biord v', \text{then } v\biord v
\end{equation}
That is, no dissonant valuation is as good as any other.
This fact will be useful in proving some results in later sections. (We note that~$\biord_\KB$ constructed in the completeness proof for Theorem~\ref{Thm:RepResultBOB} satisfies Property~(\ref{Property:BOBrevision}), as the elements in~$U_n$ are precisely those that are not dissonant in $\biord_{\KB}$.)

Even though we do not have Consistency in general, given the assumption that~$\KB$ is consistent, we do get the following restricted version:
\begin{description}
\item[P6'] If $\alpha\in\KB$, then $\KB\brev\alpha\nentails\bot$ \hfill (Relative Consistency)
\end{description}
We will return to the consistency issue later. 
We also have:
\begin{description}
\item If $\alpha\in\KB$, then $\KB\brev\alpha\subseteq \KB$\hfill(Endogenous Inclusion)
\end{description}
Nevertheless, BOB revision does not satisfy the following property, in general:

\begin{description}
\item If $\alpha\in\KB$, then $\KB\subseteq\KB\brev\alpha$\hfill(Positive Vacuity)
\end{description}
That is, it is possible that, upon being presented with a sentence that is already a consequence of the agent's beliefs, the agent may in fact remove some beliefs. This can be seen in Example~\ref{Ex:BiorderInterpretation}, where we have $\KB=\Cn(\{\neg p \wedge \neg q\})$ and hence $p \leftrightarrow q \in \KB$, but $\KB \brev (p \leftrightarrow q) = \Cn(\{p\leftrightarrow q\})$, i.e., we lose belief in $\neg p \wedge \neg q$. 




It {\em is} possible to capture Positive Vacuity by enforcing a simple condition on biorder-based interpretations, namely$\lo^{-1}(0) \cap \up^{-1}(0) \neq \emptyset$. However, doing that would still not necessarily give us the following generalisation of it, which corresponds to the rule of Cautious Monotony from conditional reasoning:

\begin{description}
\item If $\beta\in\KB\brev\alpha$, then $\KB\brev\alpha\subseteq\KB\brev(\alpha\wedge\beta)$
\end{description}
That is, acceptance of $\alpha$ may still leave us open to some kind of instability. We have the following counterexample.

\begin{example}\label{Ex:PosVacDoesntImplyCCM}
Suppose $\KB = Cn(\{p \leftrightarrow q\})$ and $\brev$ associates the biorder $\biord_\BM$ to $\KB$, where $\BM = \tuple{\lo,\up}$ is a compressed interpretation specified by: $\lo(pq)=\lo(\bar{p}\bar{q}) = 0$, $\lo(p\bar{q}) = 1$ and $\lo(\bar{p}q)=2$; $\up(\bar{p}\bar{q})= \up(\bar{p}q) = 0$, $\up(p\bar{q})=1$, and $\up(pq)=2$. Then $p \in \KB \brev (p \vee q)$ but $q \in \KB \brev (p \vee q) \setminus \KB \brev ((p \vee q) \wedge p)$.
\end{example}
In the above example, $p$ is an example of what we call a {\em precarious} input:
\begin{definition}
[Precarious Sentences]
Given a BOB revision operator $\brev$ and a belief set~$\KB$, we say that~$\alpha$ is \df{precarious} for~$\brev$ and~$\KB$ if~there is~$\beta\in\Lang$ \suchthat~$\beta\in\KB\brev\alpha$ and $\KB\brev\alpha \nsubseteq\KB\brev(\alpha\lor\beta)$
\end{definition}

We have the following easy result concerning the relationship between destabilising, irreconcilable and precarious sentences:

\begin{restatable}{proposition}
{Destabilisingimpliesothers}
\label{Prop:Destabilising_implies_others}
Let $\alpha \in \Lang$. If $\alpha$ is destabilising for $\brev$ and $\KB$, then it is either irreconcilable or precarious for \mbox{$\brev$ and $\KB$.}
\end{restatable}

\subsection{ZTBOB Revision}

We now turn our attention to a special subclass of biorders, namely those also satisfying some forms of transitivity.

We start by noting that the converse to Proposition~\ref{Prop:Destabilising_implies_others} does not hold, in general, as can easily be seen in Example \ref{Ex:PosVacDoesntImplyCCM}, where $p$ is precarious but not destabilising for $\brev$ and $\KB$. 
Since clearly irreconcilability implies destabilising, this boils down to requiring that precarious implies destabilising, which can be expressed as the following Consistent Cautious Monotony (CCM) postulate on the~$\bullet$ operator:

\begin{description}
\item[P10] If $\beta\in\KB\brev\alpha$ and $\KB\brev\alpha\nentails\bot$, then $\KB\brev\alpha\subseteq\KB\brev(\alpha \land\beta)$\hfill(Consistent Cautious Monotony)
\end{description}

We can satisfy this postulate by forcing a further condition on~$\BM$ and~$\biord_{\BM}$ (remember~$z$ denotes a dissonant valuation):
\begin{description}
\item[BZT]
If  $\lo(u)\leq\up(z)<\up(u)$,  then $\lo(z)\leq\up(z)$
\end{description}

In words, there is no dissonant valuation $z$ such that $\lo(u) \leq \up(z) < \up(u)$. This property fails for the $\BM$ in Example \ref{Ex:PosVacDoesntImplyCCM}, as can be seen by taking $u = pq$ and $z = \bar{p}q$. If~$\BM$ satisfies BZT, then $\biord_{\BM}$ also satisfies the following weakening of transitivity, which we call \emph{$z$-transitivity}:
\begin{equation}\label{zTransitivity}
\textrm{If }u\biord v,\ v\biord z\textrm{ and }z\not\biord z, \textrm{ then } u\biord z
\end{equation}

We call any biorder~$\biord$ satisfying Property~(\ref{zTransitivity}) a \emph{$z$-transitive} biorder.

\begin{restatable}{proposition}
{zTransitivityConditionOnBOB}
\label{Prop:zTransitivityConditionOnBOB}
\emph{(i)} Let $\BM = \langle \lo, \up \rangle$ be a biorder-based interpretation satisfying BZT. 
Then $\biord_\BM$ is $z$-transitive. \emph{(ii)}~Let $\biord$ be a $z$-transitive biorder. Then every \df{compressed} biorder-based interpretation $\BM=\tuple{\lo,\up}$ \suchthat\ $\biord_\BM = \biord$ satisfies BZT. 
\end{restatable}
The ordering $\biord_\BM$ in Example \ref{Ex:PosVacDoesntImplyCCM} is not $z$-transitive, as witnessed by the fact that $p\bar{q} \biord_\KB pq$, $pq \biord_\KB \bar{p}q$, $\bar{p}q \not\biord_\KB \bar{p}q$ but $p\bar{q} \not\biord_\KB \bar{p}q$. 

We thus arrive at our first subclass of BOB revision, and its associated axiomatic characterisation.

\begin{definition}[ZTBOB Revision]
$\bullet$ is a $z$-transitive biorder-based (ZTBOB) revision operator if, to each $\KB$ we can associate a $z$-transitive biorder~$\biord_{\KB}$ anchored on $\Mod{\KB}$ \suchthat, for each~$\alpha$, $\KB\brev\alpha$ is determined from~$\biord_{\KB}$ as in~(\ref{Def-BO-Revision}).
\end{definition}

\begin{restatable}{theorem}
{RepResultZTBOB}
\label{Thm:RepResultZTBOB}
An operator~$\brev$ is a ZTBOB revision operator iff it satisfies all the postulates from Theorem \ref{Thm:RepResultBOB}, \ie, Tautology, Closure, Success, Inclusion, Extensionality, Disjunctive Overlap and Disjunctive Rationality, \textbf{plus}~CCM. 
\end{restatable}

\subsection{TBOB Revision}

As discussed in the previous section, $z$-transitivity is a weakened form of transitivity. A natural question to ask is ``what if we enforce full transitivity?''  In terms of compressed biorder-based interpretations (Definition~\ref{Def:CompressedBiorderBasedInterpretation}), this corresponds to a particularly simple property, namely: 
\begin{description}
\item[BT]
$\up(v)\leq\lo(v)$
\end{description}

Notice that we still allow for valuations to be assigned negative intervals here, but we disallow intervals of non-zero positive length. As a result, transitive biorders are still not necessarily reflexive.

\begin{restatable}{proposition}
{TransitivityConditionOnBOB}
\label{Prop:TransitivityConditionOnBOB}
\emph{(i)} For every biorder-based interpretation~$\BM$ satisfying BT, $\biord_{\BM}$ is transitive. \emph{(ii)} Let~$\biord$ be a transitive biorder; then every \df{compressed} biorder-based interpretation~$\BM=\tuple{\lo,\up}$ \suchthat\ $\biord_\BM = \biord$ satisfies BT.
\end{restatable}

One can see that in the biorder-based interpretation from Example~\ref{Ex:BiorderInterpretation}, the condition BT does not hold as~$\up(\bar{p}\bar{q})>\lo(\bar{p}\bar{q})$. Furthermore, the associated biorder~$\biord_{\BM}$ is not transitive as $\{(p\bar{q},\bar{p}\bar{q}),(\bar{p}\bar{q},p\bar{q})\}\subseteq\biord_{\BM}$, but $(p\bar{q},p\bar{q})\notin\biord_{\BM}$.

\begin{definition}[TBOB Revision]
$\bullet$ is a transitive biorder-based (TBOB) revision operator if, to each $\KB$ we can associate a transitive biorder~$\biord_{\KB}$ anchored on $\Mod{\KB}$ \suchthat, for each~$\alpha$, $\KB\brev\alpha$ is determined from~$\biord_{\KB}$ as in~(\ref{Def-BO-Revision}) above.
\end{definition}

As it turns out, TBOB revision brings us a step closer to AGM revision, in that this family of operators captures the second AGM supplementary postulate, Subexpansion. We have the following representation result:

\begin{restatable}{theorem}
{RepResultTBOB}\label{Thm:RepResultTBOB}
An operator~$\brev$ is a TBOB revision operator iff it satisfies all postulates from Theorem~\ref{Thm:RepResultBOB}, \ie, Tautology, Closure, Success, Inclusion, Extensionality, Disjunctive Overlap and Disjunctive Rationality, \textbf{plus} Subexpansion. 
\end{restatable}

In other words, transitive biorders characterise the revision operators that are characterised exactly by the AGM postulates \emph{minus} Consistency \emph{plus} Disjunctive Rationality. (Note that Vacuity is recaptured since it follows from Subexpansion, Tautology and Extensionality). In the literature on rational choice theory, the family of choice functions corresponding to a restricted version of transitive biorders has been characterised by Gerasimou~\shortcite{Gerasimou2018}.


We note that Subexpansion is equivalent to the following Disjunctive Monotony (DM) postulate, given the other postulates for TBOB revision:
\begin{description}
\item
If $\lnot\beta\notin\KB\brev\alpha$,  then $\KB\brev(\alpha\lor\beta)\subseteq\KB\brev\beta$ \hfill (DM)
\end{description}

This postulate is essentially the belief revision version of  the~($\gamma^{+}$) postulate from rational choice of Salant and Rubinstein~\shortcite{SalantRubinstein2008}, which has been considered by Gerasimou~\shortcite{Gerasimou2018}, and also by Grossi~\etal.~\shortcite{GrossiEtAl2022}, from where we have taken the name Disjunctive Monotony. We could substitute~it for Subexpansion in Theorem \ref{Thm:RepResultTBOB} above, but would still need Disjunctive Rationality as it does not follow from the other postulates mentioned in Theorem \ref{Thm:RepResultTBOB}. As we will see, Disjunctive Monotony comes in handy for the representation results we state in the next section.

The property $\up(u) \leq \lo(u)$ for all $u$, combined with the property that $\lo(u)$ is irrelevant for all those $u$ such that $\up(u) < \lo(u)$ roughly means that we only need to keep $\up$. Furthermore any dissonant valuations $z$ \suchthat\ $\up(z) = i$ can be safely placed in a newly created upper rank in between $i$ and $i+1$ without changing the revision behaviour. All this means that we can specify an interpretation $\langle \lo, \up \rangle$ using a single ranking $\RM$   {\bf plus} a specification of which, if any, ranks are ``impossible''. See Figure~\ref{Fig:TBOB-Revision-Spheres} above for an illustration.
\[
\Mod{\KB \brev \alpha} =
\left\{
\begin {array}{l}
\{u \in \Mod{\alpha} \mid \RM(u) \leq \RM(\alpha)\},  \\
\qquad \qquad \qquad \textrm{if }\RM(\alpha)\ \textrm{is not impossible} \\
\emptyset,  \qquad \qquad  \textrm{otherwise}
\end{array}
\right.
\]

\begin{figure}
    \centering
    \includegraphics[width=0.7\linewidth]{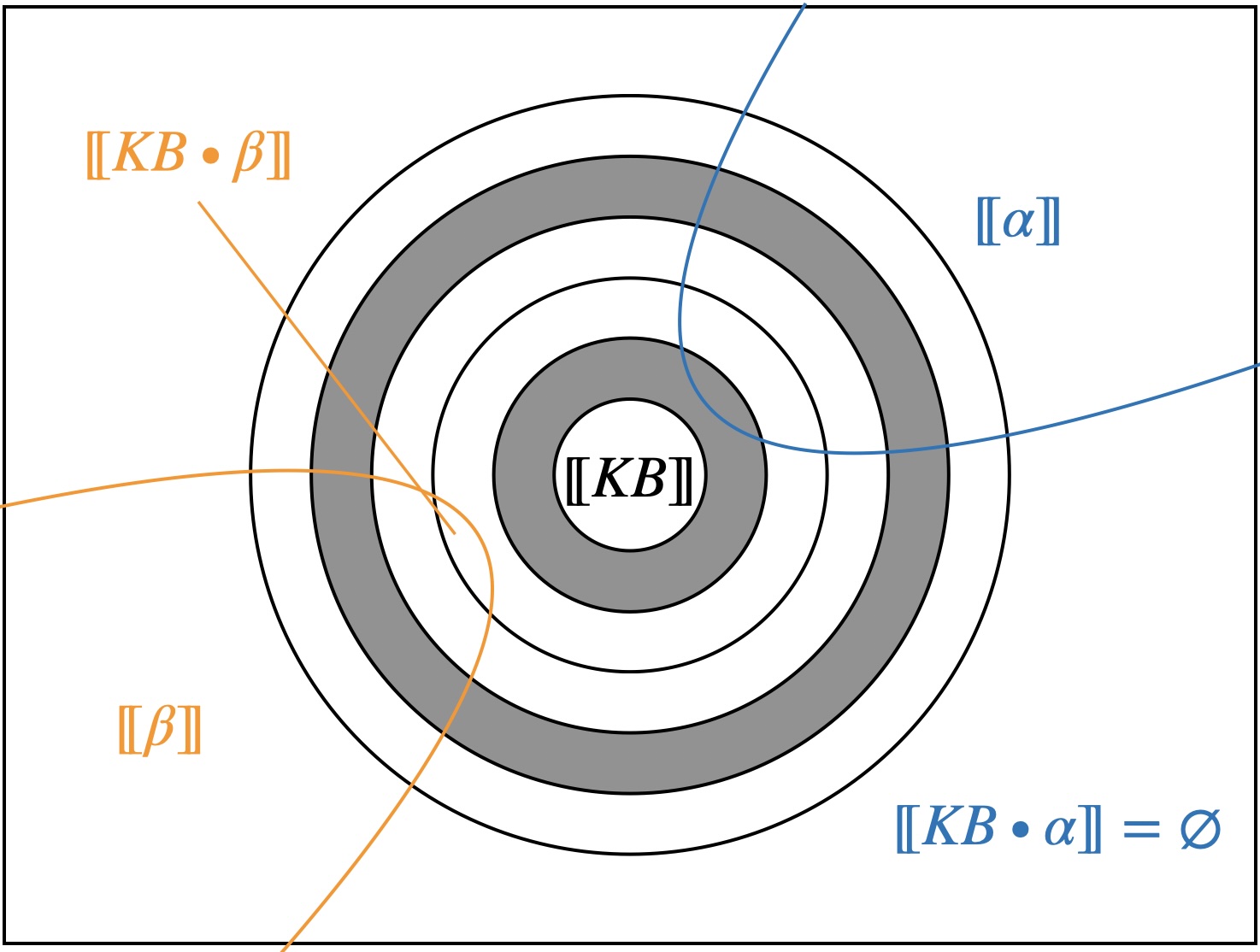}
    \caption{System-of-spheres style representation of a TBOB revision operator. Grey rings indicate ``impossible'' ranks.}
    \label{Fig:TBOB-Revision-Spheres}
\end{figure}

To sum it up, while both BOB and TBOB revision operators satisfy many important AGM postulates, in particular the Success one, their significance for standard belief revision is, at least at first sight, diminished by the fact that they do not always yield consistent outputs. Indeed, the Consistency postulate is non-negotiable in `orthodox' belief revision. We could, in principle, circumvent this by enforcing consistency. Note, however, that this would lead us back to interval orders or even total preorders, which is essentially where we started from. In the next section, we step back and look at the possibly incoherent revision inputs as still modelling something useful about an agent's belief state.

\section{Non-prioritised Revision}\label{NPRevision}

We shall now look at our previously defined operators from a different angle, namely that of non-prioritised belief revision. This entitles us to modify their definition in order to discard inputs leading to inconsistency as `incredible'. This captures the intuition that whenever an agent receives as input a sentence~$\alpha$ that is destabilising, they will just ignore it, sacrificing Success for Consistency, which is precisely the non-prioritised revision approach. 

\subsection{BOB NPR Operators}

We start by noting that the last three operators we have defined, namely general biorder-based, $z$-transitive biorder-based, and transitive biorder-based revisions, give rise to corresponding classes of non-prioritised revision operators, with associated axiomatics.

\begin{definition}[BOB Non-Prioritised Revision]\label{Def:NPRevisionFromRevision}
Let $\brev$ be a BOB revision operator. The \df{BOB non-prioritised revision operator} (\df{BOB NPR}, for short)~$\nprev$ generated by~$\brev$ is the operator defined by setting, for each belief set~$\KB$ and~$\alpha\in\Lang$,
\[
\KB\nprev\alpha\defined\left\{
\begin{array}{ll}
\KB\brev\alpha, & \text{ if this is consistent} \\
\KB, & \textrm{otherwise}
\end{array}
\right.
\]
\end{definition}
Equivalently, $\nprev$ is a BOB NPR operator if to each $\KB$ we can associate a biorder $\biord_\KB$ anchored on $\Mod{\KB}$ \suchthat\ $\Mod{\KB \nprev \alpha}$ is $\opt(\Mod{\alpha}, \biord_\KB)$ if this  is non-empty, and is $\Mod{\KB}$ otherwise.
We similarly obtain ZTBOB NPR operators and TBOB NPR operators, when~$\brev$ is taken to be a ZTBOB, respectively TBOB, revision operator. 

\begin{definition}[Credible Sentences]\label{Def:CredibleSentence}
We say~$\alpha$ is \df{credible}~\wrt~$\KB$ and~$\nprev$ (denoted $\alpha\in\CredSet^\nprev_\KB$) if~$\alpha\in\KB\nprev\alpha$, and \df{incredible} otherwise.
\end{definition}

It is easy to see that if~$\nprev$ is defined from~$\brev$ as above, then~$\alpha$ is credible~\wrt~$\KB$ and~$\nprev$ if and only if~$\alpha$ is not destabilising \wrt~$\KB$ and~$\brev$ (\cf~Definition~\ref{Def:DestabilisingSentence}).



\noindent Definition~\ref{Def:NPRevisionFromRevision} leads us to the following representation result:

\begin{restatable}{theorem}
{biordercredrevision}
\label{theorem:biorder_cred_revision}
An operator~$\nprev$ is a BOB NPR operator iff it satisfies Tautology, Closure, Inclusion, Extensionality, Consistency, Disjunctive Overlap and the list of postulates below. 
\begin{description}
\item[P2$'$] $\alpha\in\KB\nprev\alpha$ or $\KB\nprev\alpha=\KB$ \hfill (Relative Success)
\item[P11] If $\alpha\lor\beta\in\KB\nprev(\alpha\lor\beta)$, then $\alpha\in\KB\nprev\alpha$ or~$\beta\in\KB\nprev\beta$
\item[P12] If $\alpha\lor\beta\in\KB\nprev(\alpha\lor\beta)$ and $\beta\notin\KB\nprev\beta$, then $\KB\nprev\alpha\subseteq\KB\nprev(\alpha\lor\beta)$
\item[P13] If $\alpha\in\KB\nprev\alpha$ and $\beta\in\KB\nprev\beta$, then $\KB\nprev(\alpha\lor\beta)\subseteq\KB\nprev\alpha\cup\KB\nprev\beta$
\end{description}
\end{restatable}

Some of these postulates are familiar from the literature on credibility-limited revision~\cite{hansson2001credibility}. P2$'$ is a hallmark property of such revision operators. It says that, at least as far as their effects on the belief set are concerned, incredible input sentences are plain ignored. P11 says a disjunction is credible only if at least one of its disjuncts is. P12 says that if one of the disjuncts~$\beta$ of a credible sentence $\alpha\lor\beta$ is incredible, then~$\KB\nprev(\alpha\lor\beta)$ should contain everything that would be believed from revising with~$\alpha$ alone. P13 is a restricted form of Disjunctive Rationality~(DR), saying that it should hold whenever~$\alpha$ and~$\beta$ are both credible.


\begin{proposition}
Disjunctive Overlap implies the following postulate in the presence of Closure. 
\begin{description}
\item[P14] If $\alpha\in\KB\circ\alpha$ and $\beta\in\KB\circ\beta$, then $\alpha\lor\beta\in\KB\circ(\alpha\lor\beta)$
\end{description}
\end{proposition}

The above property says the disjunction of two credible sentences is itself credible.

\subsection{ZTBOB and TBOB NPR Operators}

One of the characteristic properties for the general biorders-based NPR~$\circ$ is~P12. However, a case can be made for a version of~P12 with equality in the consequent, \ie, if a credible disjunction is received, then any incredible disjunct is essentially ignored.
\begin{description}
\item[P12+] If $\alpha\lor\beta\in\KB\circ(\alpha\lor\beta)$ and $\beta\notin\KB\circ\beta$, then $\KB\circ\alpha=\KB\circ(\alpha\lor\beta)$
\end{description}

P12+ does not hold, in general, for BOB NPR operators: 

\begin{example}
Let $\KB = Cn(\{p \leftrightarrow q\})$ and take the biorder $\biord_\BM$ anchored on $\KB$ given by $\BM$ from Example \ref{Ex:PosVacDoesntImplyCCM}. We have $p \vee (\neg p \wedge q) \in \KB \nprev (p \vee (\neg p \wedge q)) = Cn(\{p \wedge q\})$ and $\neg p \wedge q \notin \KB \nprev (\neg p \wedge q)$. But $\KB \nprev p = Cn(\{p\})$ and so $\KB \nprev (p \vee (\neg p \wedge q)) \nsubseteq \KB \nprev p$. 
\end{example}

It may be noted in the previous example that BZT fails, i.e, the ordering is not $z$-transitive. It turns out that the BOB NPR operators satisfying~P12+ are {\em precisely} the ZTBOB NPR operators.  

\begin{restatable}{theorem}{ztransbiordercredrevision}
\label{theorem:z_trans_biorder_cred_revision}
An operator~$\nprev$ is a ZTBOB NPR operator iff it satisfies all the postulates from Theorem~\ref{theorem:biorder_cred_revision}, \ie, Tautology, Closure, Inclusion, Extensionality, Consistency, Disjunctive Overlap, Relative Success, P11, P13 with P12 replaced by P12+.
\end{restatable}

Notice that~P12+ causes some redundancy in the list: Besides replacing~P12, it causes~P13 to reduce to Disjunctive Rationality. So the official characteristic list of properties for~ZTBOB NPR operators becomes: Tautology, Closure, Relative Success, Inclusion, Extensionality, Consistency, Disjunctive Overlap, Disjunctive Rationality, P11 and~P12+.


When a TBOB operator~$\brev$ is assumed in the definition of~$\nprev$, the following representation result is warranted:

\begin{restatable}{theorem}
{transbiordercredrevision}\label{theorem:trans_biorder_cred_revision}
An operator~$\nprev$ is a TBOB NPR operator iff it satisfies all the postulates satisfied by ZTBOB NPR operators, \ie, Tautology, Closure, Relative Success, Inclusion, Extensionality, Consistency, Disjunctive Overlap, Disjunctive Rationality, P11 and P12+, \textbf{plus} Disjunctive Monotony. 
\end{restatable}

Note that, although Disjunctive Monotony is sound for TBOB NPR operators, Subexpansion (which {\em is} sound for BOB revision operators) is not. 

\section{BOB NPR as Credibility-limited Revision}\label{CredLimitedRevision}


In this section, we show how the two subclasses of BOB non-prioritised revision from the previous section can be formulated as credibility-limited revision operators.

Credibility-limited revision has been put forward by Hansson~\etal.~\shortcite{hansson2001credibility} and its general schema is as follows. We assume a (successful and consistent) revision operator~$\rev$ and a function $\CredSet$ that, to each belief set~$\KB$, assigns a \emph{credible set}~$\CredSet_\KB \subseteq\Lang$, which, intuitively, contains sentences deemed acceptable. Then we have the following operator:
\[
\KB\circ_{\rev,\CredSet}\alpha\defined\left\{
    \begin {array}{ll}
        \KB\rev\alpha, & \textrm{if } \alpha\in\CredSet_\KB \\
        \KB, & \textrm{otherwise}
\end{array}
\right.
\]
The operator $\rev$ is to be interpreted as 
the revision operator the agent would use if they \emph{had} to accept the new sentence (even though at first glance they would choose to reject it). The sentences not in~$\CredSet_\KB$ represent inputs that are in some sense too remote in the agent's belief state to be entertained.

By assuming different properties of~$\rev$ and~$\CredSet$, and possibly some properties governing the interplay between these two, we obtain different properties of the operator~$\circ_{\rev, \CredSet}$. Regarding the underlying revision operator~$\rev$, Hansson~\etal.\ consider the case where it is a basic or a full AGM revision operator. Regarding $\CredSet$, they consider several properties, including the following:
\begin{description}
\item[\CredSet1] $\KB \subseteq \CredSet_\KB$
\item[\CredSet2] If $\alpha\in\CredSet_\KB$ and $\alpha\equiv\beta$, then $\beta\in\CredSet_\KB$
\item[\CredSet3] If $\alpha\lor\beta\in\CredSet_\KB$, then $\alpha\in\CredSet_\KB$ or $\beta\in\CredSet_\KB$
\item[\CredSet4]
    If $\alpha \in \CredSet_\KB$ and $\alpha \entails \beta$, then $\beta \in \CredSet_\KB$
\end{description}
$\CredSet 1$ says if a sentence is already supported by our belief set, then we have no reason to reject it, $\CredSet 2$ is a syntax irrelevance property on credible sentences, while $\CredSet 3$ says a disjunction is credible only if at least one of its disjuncts is (\cf\ Postulate~P14). We will adopt $\CredSet 1$-$\CredSet 3$ here. $\CredSet 4$ is the Single Sentence Closure property~\cite{hansson2001credibility}, which we consider too strong. The reason is that it could be rational for an agent to reject~$\alpha$ at first glance (\ie, $\alpha\notin\CredSet_\KB$), but to accept it when strengthened with additional `explanation' $\alpha\land\beta$ (\ie, $\alpha\land\beta\in\CredSet$). We have already encountered a concrete example of this situation in the Introduction, where~$\alpha$ is ``my five-year-old nephew had lunch with King Charles today'' and~$\beta$ is ``King Charles visited a local school today.'' A further discussion on the reasons to reject~$\CredSet 4$ is provided by Schwind~\shortcite{Schwind2025}. We are interested in weakening Single Sentence Closure to the following:
\begin{description}
\item[\CredSet4$'$] If $\alpha,\beta\in\CredSet_\KB$, then $\alpha\lor\beta\in\CredSet_\KB$    
\end{description}
In the presence of $\CredSet 2$ the above property implies the following:
\[
\textrm{If }
\beta \wedge \alpha, \neg\beta \wedge \alpha \in \CredSet_\KB\
\textrm{then }
\alpha \in \CredSet_\KB.
\]
This seems to be plausible, after all, if we do not hold ``My five year old nephew had lunch with King Charles today'' to be credible, but we do hold ``King Charles visited a local school today and my five year old nephew had lunch with King Charles today'' to be credible, then presumably we would hold ``King Charles did not visit a local school today and my five year old nephew had lunch with King Charles today'' to be incredible.

We also require some minimal property of coherence between~$\rev$ and~$\CredSet$. Intuitively, if a sentence $\alpha$ is acceptable but the addition of $\beta$ renders it unacceptable, then $\beta$ is exceptional within $\alpha$. This is formalised by the following: 
\begin{equation}
\label{condition_on_star+credset}
\textrm{If } \alpha \in \CredSet_\KB\
\textrm{and } \alpha \wedge \beta \notin \CredSet_\KB\
\textrm{then }
\neg\beta \in \KB \rev \alpha
\end{equation}

\begin{definition}[CL Structure]
\label{Def:CLstructure}
We say the pair $\tuple{\rev,\CredSet}$ is a \df{CL structure} if {\em (i)}$\rev$ is a revision operator satisfying Success and Consistency, {\em (ii)} $\CredSet$ satisfies $\CredSet 1$-$\CredSet 3$ and $\CredSet 4'$, and {\em (iii)} $\rev$ and $\CredSet$ jointly satisfy (\ref{condition_on_star+credset}) above.
\end{definition}

First we turn to ZTBOB non-prioritised revision.

\begin{restatable}{theorem}
{ZTBOBasCL}\label{theorem:ZTBOBasCL}
An operator~$\circ$ is a ZTBOB NPR operator iff $\nprev = \nprev_{\rev, \CredSet}$ for some CL structure $\tuple{\rev,\CredSet}$ \suchthat\ $\rev$ is an IOB revision operator.
\end{restatable}

The `only if' part of the proof of this result requires us to construct a CL structure $\tuple{\rev^\nprev, \CredSet^\nprev}$ from a given BOB NPR operator $\nprev$. The definition of $\CredSet^\nprev$ is the expected one, viz.,
\[
\alpha \in \CredSet^\nprev_\KB\
\textrm{iff }
\alpha \in \KB \circ \alpha
\]
For $\rev^\nprev$ it essentially boils down to finding the appropriate way to {\em transform} a biorder into an interval order. For suppose $\nprev$ associates $\biord_\KB$ to $\KB$. We define an ordering $\intord_\KB$ as follows, for any $u, v \in \U$:
\[
\begin{array}{rcl}
u \intord_\KB v &
\textrm{iff } &
[u, v \notin \mathit{Diss}(\biord_\KB)\ \textrm{and } u \biord_\KB v] \\
& & \textrm{or } 
v \in \mathit{Diss}(\biord_\KB)
\end{array}
\]
where $\mathit{Diss}(\biord_\KB)$ is the set of valuations that are dissonant for $\biord_\KB$. It can then be shown that $\intord_\KB$ indeed forms an interval order that, provided $\biord_\KB$ is $z$-transitive (which it will be if $\nprev$ is a ZTBOB NPR operator), preserves all non-empty optimal sets. Hence, by associating each $\KB$ with $\intord_\KB$, we obtain an IOB revision operator $\rev^\nprev$ and we get $\nprev = \nprev_{\rev, \CredSet}$ for $\rev = \rev^\nprev$ and $\CredSet = \CredSet^\nprev$.  

If $\biord_\KB$ is not merely $z$-transitive but fully transitive, then $\intord_\KB$ above becomes a total preorder, which provides the key extra step towards the following characterisation of TBOB NPR operators.
\begin{restatable}{theorem}
{TBOBasCL}
An operator~$\circ$ is a TBOB NPR operator iff $\nprev = \nprev_{\rev, \CredSet}$ for some CL structure $\tuple{\rev,\CredSet}$ \suchthat\ $\rev$ is an AGM revision operator.
\end{restatable}



\section{Related Work}\label{RelatedWork}

As already mentioned, Rott~\shortcite{Rott2014} was the first to have considered (strict) interval orders in belief change to define belief \emph{contraction} operators. Booth and Varzinczak~\shortcite{BoothVarzinczak2021-AAAI,BoothVarzinczak2025-AIJ} have utilised such orderings as the semantic foundation for defeasible conditionals, leading to the identification of a relevant notion of non-monotonic consequence relation lying between preferential closure and rational closure~\cite{LehmannMagidor1992}.

An investigation of the interdefinability of Rott's contraction operators and our IOB revision via the Levi and Harper identities is an exercise we shall leave for future work.

In addition to the work of Rott, there have been several other studies in belief change that deploy kinds of ordering that are different from total preorders, starting with Katsuno and Mendelzon~\shortcite{KatsunoMendelzon1991}, who considered using a slightly restricted form of strict partial order. Their work was generalised by Benferhat~\etal.~\shortcite{BenferhatEtAl2005}, who used unrestricted strict partial orders. Semiorders, which are a specific kind of (strict) interval order which correspond to interval-based interpretations in which all non-zero length intervals have the same length, have been considered for belief contraction by Rott~\shortcite{Rott2014} and for both revision and contraction by Peppas and Williams~\shortcite{PeppasWilliams2014}. The latter considered a slightly restricted family of semiorders designed to ensure the Vacuity postulate held. This restriction was lifted by Le\'on and Pino P\'erez \shortcite{LeonPinoPerez202}. Finally, Gaigne~\etal.~\shortcite{GaigneBelahceneLagrue2024} recently looked at a novel class of orderings called partitioned linear orders. All of the above papers offer representation results for the resulting classes of revision operators. 

With regard to credibility-limited revision, to the best of our knowledge, Schwind~\shortcite{Schwind2025} was the first to have challenged the single-sentence closure condition in this setting. Building on Booth~\etal.'s~\shortcite{DBLP:conf/kr/BoothFKP12} work, he motivates and introduces a family of \emph{context-based} (CB) revision operators. These capture the intuition that a sentence may be deemed credible on the grounds that it strengthens (or provides extra context for) one of its non-credible consequences.

In terms of postulates, Schwind's proposal consists of replacing Success Monotonicity in Booth~\etal.'s original proposal, which formalises the single-sentence closure condition, with a set of weaker postulates. Semantically, the belief revision operators are defined in terms of total preorders, as in the standard AGM approach to belief change, with the difference that the orderings are informed by a further refinement of the set of possible valuations.

Besides the fact that here we consider alternatives to total preorders in belief revision, our approach to credibility-limited revision differs from Schwind's in other key aspects. First, the CB revision with an input~$\alpha$ that is consistent with the agent's initial beliefs systematically reduces to belief \emph{expansion}, \ie, the result of revision is $\Cn(\KB\cup\{\alpha\})$. This corresponds to the Vacuity~(P4) postulate, which, as we have seen in Example~\ref{Ex:FailureVacuitySubexpansion}, does not hold in general in our approach, but which can be enforced if desired, via the condition in~(\ref{ConditionForVacuity}). In that sense, our construction offers the user a choice.

Second, one of the core postulates in Schwind's proposal is the Contextual Monotonicity (CM) principle, which states that if strengthening a credible sentence results in a non-credible one, then every further strengthening thereof remains incredible. We believe this is a rigid stricture to impose on the behaviour of the revision operator, as the following example, extending the one by Schwind~\shortcite{Schwind2025}, shows.

\begin{example}\label{Ex:UnintuitivenessOfContextualMonotony}
Let~$p$ be ``Bob had lunch with the prime minister yesterday'' (initially non-credible), let~$q$ be ``the prime minister visited Bob's institute'' (and then~$p\land q$ becomes credible), and let~$r$ be ``Bob had dinner with the prime minister yesterday'' (so that $p\land q\land r$ is non-credible). In Schwind's approach, every~$\alpha\in\Lang$ \suchthat\ $\alpha\entails p\land q\land r$ remains non-credible. However, let~$s$ be ``they've discussed dinner plans over lunch'', so that $p\land q\land r\land s$ ought to be credible.
\end{example}

It can be shown that the principles underlying Schwind's Contextual Monotonicity do not hold in our framework, in general, and that our~CL Structures from Definition~\ref{Def:CLstructure} can handle cases such as the one in Example~\ref{Ex:UnintuitivenessOfContextualMonotony} correctly, as already suggested by Figure~\ref{Fig:TBOB-Revision-Spheres}.


\section{Concluding Remarks}\label{Conclusion}

Recapitulating the main contributions of the present paper, they can be summarised as follows: (\emph{i})~the definition of a family of (successful) belief revision operators, each one informed by a special type of \emph{biorder}, something that had not been sufficiently explored in the belief-change literature before, (\emph{ii})~the provision of axiomatic characterisations \emph{à la} AGM for the respective operators, along with an analysis of the postulates characterising each of them, in particular of the conditions sufficient for enforcing those that do not hold in each of the new settings, (\emph{iii})~the deployment of the biorder-based revision operators in non-prioritised revision for recovering consistency and a provision of the respective representation results, and (\emph{iv})~the formulation of biorder-based non-prioritised revision operators as credibility-limited revision with an accompanying interdefinability result, delivering an intuitive and versatile relaxation of the Single Sentence Closure principle.

Among the consequences of the above contributions to knowledge representation and reasoning is the fact that potential users now have a larger set of (justifiable) choices for belief revision at their disposal. Furthermore, the utilisation of biorders in belief revision opens up a broad avenue for exploration in closely-related areas, \eg~in reasoning with defeasible conditionals and in normative reasoning.

Immediate topics for further investigation stemming from the results in this paper include: (\emph{i})~the identification of a meaningful strengthening as well as of an interesting weakening of the postulates on the credible set~$\mathcal{C}$, (\emph{ii})~a formulation of \emph{general} BOB revision, besides the ZTBOB and TBOB ones, into the framework of credibility-limited revision, and (\emph{iii})~the definition and study of BOB, ZTBOB and TBOB \emph{contraction} and of how these stand in comparison to Rott's contraction operators.

A further research question relates to the use of the full expressivity of biorders. As we have seen, different biorders can yield the same revision behaviour. We conjecture the answer to where this extra structure comes from may emerge from an investigation of the iterated revision case.


\section*{Acknowledgments}

Thanks are due to Xavier Parent for interesting discussions and to the anonymous reviewers for their useful comments.

\section*{AI Declaration}

The authors used Copilot to suggest some of the terminology used in the paper, and to check for any mistakes in the proofs once they had been written up.


\bibliographystyle{kr}
\bibliography{References}

\clearpage
\setcounter{page}{1}

\section*{Appendix: Proofs}

\RepResultIOB*

\begin{proof}
The proof actually extends that given for the characterisation of the more general BOB revision operators (Theorem \ref{Thm:RepResultBOB}).

\noindent 
{\em (Soundness)}
Since every interval order is a biorder, we know every IOB revision operator is a BOB revision operator, and so Theorem \ref{Thm:RepResultBOB} already gives us soundness for all the postulates other than Consistency. To show the latter suppose $\intord_\KB = \intord_\IM$ for some interval-based interpretation $\IM = \tuple{\lo,\up}$. Since $\lo(\alpha) \leq \up(\alpha)$ we know $\Mod{\KB \ast \alpha} \neq \emptyset$ if $\Mod{\alpha} \neq \emptyset$ as required. 

\noindent
{\em (Completeness)}
Suppose $\ast$ satisfies the stated postulates. By Theorem \ref{Thm:RepResultBOB} we know that, for each $\KB$, $\KB \ast \alpha$ is determined from $\biord_{\BM}$ where $\BM = \BM^\ast_\KB$ is the canonical biorder-based interpretation associated to $\ast$ and $\KB$. It thus suffices to show that the extra assumption of Consistency forces $\BM^\ast_\KB$ to be an interval-based interpretation, i.e., that $\lo^\ast_\KB(u) \leq \up^\ast_\KB(u)$ for all $u \in \U$. But if $\up^\ast_\KB(u') < \lo^\ast_\KB(u')$ for some $u'$ then we would have $\Mod{\KB \ast \alpha} = \emptyset$, where $\alpha$ is \suchthat\ $\Mod{\alpha} = \{u'\}$ - contradicting Consistency as required.
\end{proof}

\RepresentationOfBiorders*

\begin{proof}
Given a biorder $\biord$ and $u \in \U$ we use $up_\biord(u)$ and $down_\biord(u)$ to denote the {\em up-set} and {\em down-set} of $u$ respectively, i.e., $up_\biord(u) = \{v \in \U \mid u \biord v\}$ and $down_\biord(u) = \{v \in \U \mid v \biord u\}$. It is well-known (e.g., \cite{Fishburn1985}) that the Ferrers condition is equivalent to saying that the collection of all down-sets of elements from $\U$ is nested, which in turn is equivalent to saying the collection of all up-sets of elements from $\U$ is nested. 

To get a biorder-based interpretation from $\biord$ we define two increasing sequences of sets of worlds inductively by setting $U_{-1} = V_{-1} = \emptyset$ and then, for each $i \geq -1$:
\[
U_{i+1} = U_{i} \cup \{u \in \U \mid V_i^c \subseteq up_\biord(u)\}
\]
\[
V_{i+1} = \{u \in \U \mid down_\biord(u) \subseteq U_{i+1}\}
\]
We let $n$ be minimal such that both $U_n = \U$ and $V_n = \U$. We need to show that the $U_i$ and the $V_i$ are {\em strictly} increasing, which ensures that $n$ exists. This is ensured by the following properties: {\em (i)} $V_1 \neq \emptyset$, and {\em (ii)} for $i \geq 1$ we have ($U_i \neq \U \Rightarrow U_{i+1} \nsubseteq U_i$) and ($V_i \neq \U \Rightarrow V_{i+1} \nsubseteq V_i$).

\noindent
\underline{{\em (i) $V_1 \neq \emptyset$}}
\\
For $i = 1$ we have $U_1 = \opt(\U, \biord)$ and $V_1 = \{u \in \U \mid down_\biord(u) \subseteq U_1\}$. 
Note we can equivalently write $U_1 = \bigcap \{down_\biord(u) \mid u \in \U\}$ and so, since the down-sets are nested, $U_1 = down_\biord(u')$ for some $u' \in \U$. But then $u' \in V_1$ by definition of $V_1$ and so $V_1 \neq \emptyset$.

\noindent
\underline{{\em (ii)} for $i \geq 1$ ($U_i \neq \U \Rightarrow U_{i+1} \nsubseteq U_i$) and ($V_i \neq \U \Rightarrow V_{i+1} \nsubseteq V_i$)}
\\
Assume $U_i \neq \U$, i.e., $U_i^c \neq \emptyset$. By definition of $V_i$, for each $u \in V_i^c$ there is some $v \in U_i^c$ \suchthat\ $v \biord u$, so we have $V_i^c \subseteq \bigcup\{up_\biord(v) \mid v \in U_i^c\}$. Since $U_i^c \neq \emptyset$ and the up-sets are nested, there is~$v' \in U^c_i$ \suchthat\ $up_\biord(v') = \bigcup\{up_\biord(v) \mid v \in U_i^c\}$ and so $V_i^c \subseteq up_\biord(v')$. Then, by definition of $U_{i+1}$, we have $v' \in U_{i+1}\setminus U_i$ as required.
\\
\noindent
Assume $V_i \neq \U$, i.e., $V_i^c \neq \emptyset$. By definition of $U_{i+1}$, we have $\{u \in \U \mid V_i^c \subseteq up_\biord(u)\} \subseteq U_{i+1}$, so, since $\{u \in \U \mid V_i^c \subseteq up_\biord(u)\} = \bigcap\{down_\biord(u) \mid u \in V_i^c\}$, we have $\bigcap\{down_\biord(u) \mid u \in V_i^c\} \subseteq U_{i+1}$. Since $V_i^c \neq \emptyset$ and the down-sets are nested, there is~$u' \in V^c_i$ \suchthat\ $down_\biord(u') = \bigcap\{down_\biord(u) \mid u \in V_i^c\}$ and so $down_\biord(u') \subseteq U_{i+1}$. Then, by definition of $V_{i+1}$, we have $u' \in V_{i+1}\setminus V_i$ as required.

We then define $\lo$ and $\up$ by setting, for each $u \in \U$, $\lo(u)$ to be the minimal $i$ \suchthat\ $u \in U_i$, and $\up(u)$ to be the minimal $i$ \suchthat\ $u \in V_i$.

\noindent
\underline{$u \biord v$ implies $\lo(u) \leq \up(v)$}
\\
Assume $u \biord v$ and let $\lo(u) = i$. We must show $v \notin V_{i-1}$, i.e., $down_\biord(v) \nsubseteq U_{i-1}$. But $u \biord v$ and, by the minimality of $i$, $u \notin U_{i-1}$.

\noindent
\underline{$\lo(u) \leq \up(v)$ implies $u \biord v$}
\\
Assume $\lo(u) \leq \up(v)$ and let $\lo(u) = i$. Then $u \in U_i \setminus U_{i-1}$ which, by construction, means $V_{i-1}^c \subseteq up_\biord(u)$. Since $i \leq \up(v)$ we know $v \in V^c_{i-1}$ and so $v \in up_\biord(u)$, i.e., $u \biord v$.
\end{proof}

The above constructed sequences $U_i$ and $V_i$ also satisfy the following properties, which ensure the resulting $\BM$ is compressed: {\em (i)} If $V_i = \U$ then $\U_{i+1} = \U$, {\em (ii)} If $U_i = \U$ then $V_i = \U$. 

\RepResultBOB*
\begin{proof}
{\em (Outline)}
The soundness part is straightforward. For the completeness part, let $\brev$ satisfy the stated postulates. Let $\KB$ be a belief set and $\alpha \in \Lang$. We show $\opt(\Mod{\alpha}, \biord_{\KB,\brev}) = \Mod{\KB \brev \alpha}$. We have $\opt(\Mod{\alpha}, \biord_{\KB,\brev}) = \{u \in \U \mid \lo^\brev_\KB(u) \leq \up^\brev_\KB(\alpha)\} = U_i \cap \Mod{\alpha}$, where $i$ is minimal \suchthat\ $\Mod{\alpha} \nsubseteq V_i$. 

\noindent
\underline{$\opt(\Mod{\alpha}, \biord_{\KB,\brev}) \subseteq \Mod{\KB \brev \alpha}$}
Let $i$ be minimal \suchthat\ $\Mod{\alpha} \nsubseteq V_i$. By construction, $U_i = \bigcup_{j<i} \Mod{\KB \brev \bigvee V_j}$. So we must show $\Mod{\KB \brev \bigvee V_j} \cap \Mod{\alpha} \subseteq \Mod{\KB \brev \alpha}$ for all $j<i$. By Superexpansion ($\KB \brev (\alpha \wedge \beta) \subseteq Cn(\KB \brev \alpha \cup \{\beta\})$\footnote{This postulate is equivalent to Disjunctive Overlap, given the other postulates in the statement of the theorem.}) we know $\Mod{\KB \brev \bigvee V_j} \cap \Mod{\alpha} \subseteq \Mod{\KB \brev (\alpha \wedge \bigvee V_j)}$. But, by the minimality of $i$ we also know $\Mod{\alpha} \subseteq V_j$, i.e., $\Mod{\alpha \wedge \bigvee V_j} = \Mod{\alpha}$. Hence, by Extensionality, $\Mod{\KB \brev \bigvee V_j} \cap \Mod{\alpha} \subseteq \Mod{\KB \brev \alpha}$ as required.

\noindent
\underline{$\Mod{\KB \brev \alpha} \subseteq \opt(\Mod{\alpha}, \biord_{\KB,\brev})$}
Let $i$ be minimal \suchthat\ $\Mod{\alpha} \nsubseteq V_i$. We must show $\Mod{\KB \brev \alpha} \subseteq U_i \cap \Mod{\alpha}$. Clearly $\Mod{\KB \brev \alpha} \subseteq \Mod{\alpha}$ by Success, so it remains to show $\Mod{\KB \brev \alpha} \subseteq U_i$. But this follows from $\Mod{\alpha} \nsubseteq V_i$ and the definition of $V_i$ in the construction.
\end{proof}

We later make use of the following property of the canonical biorder-based interpretation for $\brev$ and $\KB$.
\begin{lemma}
\label{Lemma:Canonical_model_is_consonant} 
Let $w, u\in \U$ be \suchthat\ $\lo^\brev_\KB(w) \leq \up^\brev_\KB(u)$. Then $\lo^\brev_\KB(w) \leq \up^\brev_\KB(w)$.
\end{lemma}

\begin{proof}
Suppose $\lo^\brev_\KB(w) = i \leq \up^\brev_\KB(u) = j$. Then, since $j \leq n$ we have $i \leq n$ and so $w \in \Mod{\KB \brev \bigvee V_{i-1}}$. Hence, by Success for $\brev$, $w \in V_{i-1}$ which means $i \leq \up^\brev_\KB(w)$ as required.
\end{proof}

\Destabilisingimpliesothers*

\begin{proof}
Suppose $\alpha$ is destabilising for $\KB$ but is not irreconcilable. Then $\bot \in \KB \brev \alpha$, but $\bot \notin \KB \brev \beta$ for some $\beta$ such that $\beta \entails\alpha$. Then $\bot \in \KB \brev \alpha \setminus \KB \brev (\alpha \wedge \beta)$ and so $\alpha$ is precarious.
\end{proof}

\zTransitivityConditionOnBOB*

\begin{proof}
{\em (i)}
Suppose $u \biord_\BM v$, $v \biord_\BM z$ and $z \not\biord_\BM z$, i.e., $\lo(u) \leq \up(v)$, $\lo(v) \leq \up(z)$ and $\up(z) < \lo(z)$. From the last two conditions and BZT we have $\up(v) \leq \up(z)$. From this and $\lo(u) \leq \up(v)$, we conclude $\lo(u) \leq \up(z)$, i.e., $u \biord_\BM z$ as required.   
\\
{\em (ii)}
Let $\BM = \langle \lo, \up \rangle$ be a compressed biorder-based interpretation \suchthat\ $\biord_\BM$ is $z$-transitive. Suppose $\lo(u) \leq \up(z) < \up(u)$. From $\lo(u) \leq \up(z)$ we get $u \biord_\BM z$. Since $\BM$ is compressed, there exist $v \in \U$ \suchthat\ $\lo(v) = \up(u)$. Then $v \biord_\BM u$, while also $\up(z) < \lo(v)$ and so $v \not\biord_\BM z$. From $v \biord_\BM u$, $u \biord_\BM z$ and $v \not\biord_\BM z$ we obtain $z \biord_\BM z$ by $z$-transitivity, i.e., $\lo(z) \leq \up(z)$ as required.
\end{proof}

\RepResultZTBOB*
\begin{proof}
{\em (Soundness)}
Suppose $\brev$ is a ZTBOB revision operator which associates a $z$-transitive biorder $\biord_\KB$ to each $\KB$. By Theorem \ref{Thm:RepResultBOB} we only need to show $\brev$ satisfies Consistent Cautious Monotony. Let $\BM = \langle \lo, \up \rangle$ be a compressed biorder-based interpretation for $\biord_\KB$. By Proposition \ref{Prop:zTransitivityConditionOnBOB} we know $\BM$ satisfies BZT. Suppose $\KB \brev \alpha$ is consistent, $\beta \in \KB \brev \alpha$ and, for contradiction $\KB \brev \alpha \nsubseteq \KB \brev (\alpha \wedge \beta)$. Let $z \in \Mod{\alpha}$ be \suchthat\ $\up(z) = \up(\alpha)$. Since $\KB \brev \alpha$ is consistent we know there is some $u \in \Mod{\KB \brev \alpha}$ \suchthat\ $\lo(u) \leq \up(z)$. Since $\KB \brev \alpha \nsubseteq \KB \brev (\alpha \wedge \beta)$, i.e., $\Mod{\KB \brev (\alpha \wedge \beta)} \nsubseteq \Mod{\KB \brev \alpha}$ we know $\up(\alpha) < \up(\alpha \wedge \beta)$. From $\beta \in \KB \brev \alpha$ we know $u \in \Mod{\beta}$ so $u \in \Mod{\alpha \wedge \beta}$ and so $\up(\alpha \wedge \beta) \leq \up(u)$. Hence $\up(z) = \up(\alpha) < \up(u)$. From BZT we then infer $\lo(z) \leq \up(z)$ and so $z \in \Mod{\KB \brev \alpha}$. From this and $\beta \in \KB \brev \alpha$ we get $z \in \Mod{\beta}$. But then $\up(\alpha) = \up(\alpha \wedge \beta)$ - contradiction. Hence $\KB \brev \alpha \subseteq \KB \brev (\alpha \wedge \beta)$ as required.

\noindent
{\em (Completeness)}
Suppose $\brev$ satisfies all the postulates from Theorem \ref{Thm:RepResultBOB} plus Consistent Cautious Monotony. By Theorem \ref{Thm:RepResultBOB} we know that, for each $\KB$, $\KB \brev \alpha$ is determined from $\biord_{\BM}$ where $\BM = \BM^\brev_\KB$ is the canonical biorder-based interpretation associated to $\brev$ and $\KB$.
We show that this $\BM^\brev_\KB$ satisfies property $BZT$. For suppose for contradiction BZT doesn't hold. Then for some $u, z \in \U$ we have $\lo^\brev_\KB(u) \leq \up^\brev_\KB(z) < \up^\brev_\KB(u)$ and $\up^\brev_\KB(z) < \lo^\brev_\KB(z)$. Since $\BM^\brev_\KB$ is compressed we know there is~$w \in \U$ \suchthat\ $\lo^\brev_\KB(w) = \up^\brev_\KB(u)$ and, by Lemma \ref{Lemma:Canonical_model_is_consonant} $\lo^\brev_\KB(w) \leq \up^\brev_\KB(w)$. Now let $\alpha, \beta$ be \suchthat\ $\Mod{\alpha} = \{u, w, z\}$ and $\Mod{\beta}= \{u,w\}$. Then $\Mod{\KB \brev \alpha} = \{u\}$ and $\Mod{\KB \brev (\alpha \wedge \beta)} = \{u,w\}$. Hence $\KB \brev \alpha$ is consistent, $\beta \in \KB \brev \alpha$ and $\KB \brev \alpha \nsubseteq \KB \brev (\alpha \wedge \beta)$ - contradicting Consistent Cautious Monotony. Hence BZT must hold for $\BM$ as required.
\end{proof}

\TransitivityConditionOnBOB*

\begin{proof}
{\em (i)}
Suppose $\lo(u) \leq \up(v)$ and $\lo(v) \leq \up(w)$. Since $\up(v) \leq \lo(v)$ by BT we get $\lo(u) \leq \up(w)$ as required.

\noindent
{\em (ii)}
Let $\BM = \langle \lo, \up \rangle$ be a compressed biorder-based interpretation \suchthat\ $\biord_\BM$ is transitive and let $u \in \U$. We must show $\up(u) \leq \lo(u)$. Suppose for contradiction $\lo(u) < \up(u)$. In particular this implies $u \biord_\BM u$. Since $\BM$ is compressed there exist $v, w \in \BM$ \suchthat\ $\lo(v) = \up(u)$ and $\up(w) = \lo(u)$, which give $v \biord_\BM u$ and $u \biord_\BM w$ respectively. Hence, by transitivity, $v \biord_\BM w$, i.e., $\lo(v) \leq \up(w)$. But this gives $\up(u) \leq \lo(u)$, leading to a contradiction.
\end{proof}

\RepResultTBOB*

\begin{proof}
{\em (Soundness)}
Suppose $\brev$ is a TBOB revision operator which associates a transitive biorder $\biord_\KB$ to each $\KB$.
By Theorem \ref{Thm:RepResultBOB} we only need to show $\brev$ satisfies Subexpansion. Let $\BM = \langle \lo, \up \rangle$ be a compressed biorder-based interpretation for $\biord_\KB$. By Proposition \ref{Prop:TransitivityConditionOnBOB} we know $\BM$ satisfies BT. Suppose $\neg\alpha \notin \KB \brev \beta$, i.e., $\Mod{\KB \brev \beta} \cap \Mod{\alpha} \neq \emptyset$. Let $u \in \Mod{\KB \brev \beta} \cap \Mod{\alpha}$. Then since $u \in \Mod{\KB \brev \beta}$ we know $u \in \Mod{\beta}$ and $\lo(u) \leq \up(\beta)$. We must show $\Cn(\KB\brev\beta\cup\{\alpha\})\subseteq\KB\brev(\alpha\land\beta)$, equivalently $\Mod{\KB \brev (\alpha \wedge \beta)} \subseteq \Mod{\KB \brev \beta} \cap \Mod{\alpha}$. So let $v \in \Mod{\KB \brev (\alpha \wedge \beta)}$. Then $v \in \Mod{\alpha \wedge \beta}$ (giving $v \in \Mod{\alpha}$) and $\lo(v) \leq \up(\alpha \wedge \beta)$. Since $u \in \Mod{\alpha \wedge \beta}$ we know $\up(\alpha \wedge \beta) \leq \up(u)$. By BT $\up(u) \leq \lo(u)$ and so $\up(\alpha \wedge \beta) \leq \lo(u)$. But we already know $\lo(u) \leq \up(\beta)$ and so $\up(\alpha \wedge \beta) \leq \up(\beta)$. Combining this with $\lo(v) \leq \up(\alpha \wedge \beta)$ gives $\lo(v) \leq \up(\beta)$ and so $v \in \Mod{\KB \brev \beta}$ as required.

\noindent
{\em (Completeness)}
Suppose $\brev$ satisfies all the postulates from Theorem \ref{Thm:RepResultBOB} plus Subexpansion. By Theorem \ref{Thm:RepResultBOB} we know that, for each $\KB$, $\KB \brev \alpha$ is determined from $\biord_{\BM}$ where $\BM = \BM^\brev_\KB$ is the canonical biorder-based interpretation associated to $\brev$ and $\KB$.
We show $\BM^\brev_\KB$ satisfies property BT. For suppose for contradiction BT doesn't hold, i.e., there is~$u \in \U$ \suchthat\ $\lo^\brev_\KB(u) < \up^\brev_\KB(u)$. Since $\BM^\brev_\KB$ is compressed we know there exist $v,w \in \U$ \suchthat\ $\up^\brev_\KB(v) = \lo^\brev_\KB(u)$ and $\lo^\brev_\KB(w) = \up^\brev_\KB(u)$. From the latter we also know $\lo^\brev_\KB(w) \leq \up^\brev_\KB(w)$ from Lemma \ref{Lemma:Canonical_model_is_consonant}. Let $\alpha, \beta \in \Lang$ be \suchthat\ $\Mod{\alpha} = \{u,w\}$ and $\Mod{\beta} = \{u,v,w\}$. Then $\Mod{\KB \brev \beta} = \{u\}$ and $\Mod{\KB \brev (\alpha \wedge \beta)} = \{u,w\}$. Hence $\Mod{\KB \brev \beta} \cap \Mod{\alpha} \neq \emptyset$ but $\Mod{\KB \brev (\alpha \wedge \beta)} \nsubseteq \Mod{\KB \brev \beta} \cap \Mod{\alpha}$ -- contradicting Subexpansion.
\end{proof}

Before proving Theorem \ref{theorem:biorder_cred_revision} we need the following lemma.

\begin{lemma}\label{lemma:circ_equals_ast}
    Let $\nprev$ be a BOB NPR operator derived from a BOB revision operator $\brev$. If $\alpha \in \KB \nprev \alpha$ then $\KB \nprev \alpha = \KB \brev \alpha$.
\end{lemma}

\begin{proof}
Suppose $\alpha \in \KB \nprev \alpha$. The conclusion will follow if we can show this implies that $\KB \brev \alpha$ is consistent. So suppose for contradiction $\KB \brev \alpha$ is inconsistent. Then $\KB \nprev \alpha = \KB$ and so $\alpha \in \KB$. But $\brev$ satisfies Relative Consistency (postulate P6$'$ in Section \ref{Sect:BOB_revision}) according to which this implies $\KB \brev \alpha$ is consistent - contradiction. Hence $\KB \brev \alpha$ is consistent, as required. 
\end{proof}

\biordercredrevision*

\begin{proof}
{\em (Soundness)}
Tautology, Closure, Extensionality and Consistency are all obvious. To prove soundness of the rest we use the fact that $\KB \nprev \alpha = \KB \brev \alpha$ if the latter is consistent, and then use the already established properties of $\brev$. Inclusion follows from the fact that $\brev$ satisfies Inclusion. Relative Success follows from the fact that $\brev$ satisfies Success.

For P11 suppose $\alpha \notin \KB \nprev \alpha$ and $\beta \notin \KB \nprev \beta$. Then, by Success for $\brev$, it must be the case that both of 
$\KB \brev \alpha, \KB \brev \beta$ are inconsistent. Disjunctive Overlap and Closure for $\brev$ then give us $\KB \brev (\alpha \vee \beta)$ is inconsistent. If it were the case that $\alpha \vee \beta \in \KB \nprev (\alpha \vee \beta)$ then Lemma \ref{lemma:circ_equals_ast} would give us $\KB \nprev (\alpha \vee \beta) = \KB \brev (\alpha \vee \beta)$ and so $\KB \nprev (\alpha \vee \beta)$ is inconsistent. But this would contradict the established fact that $\nprev$ satisfies Consistency. Hence $\alpha \vee \beta \notin \KB \nprev (\alpha \vee \beta)$ as required.

For P12 suppose $\alpha \vee \beta \in \KB \nprev (\alpha \vee \beta)$ and $\beta \not\in \KB \nprev \beta$. From the latter and the fact that $\brev$ satisfies Success we know $\KB \brev \beta$ is inconsistent and so $\KB \nprev \beta = \KB$. By Disjunctive Overlap for $\brev$ we know $\KB \brev \alpha \cap \KB \brev \beta \subseteq \KB \brev (\alpha \vee \beta)$ and so, since $\KB \brev \beta = \Lang$ we have $\KB \brev \alpha \subseteq \KB \brev (\alpha \vee \beta)$. We just need to show that we can replace $\brev$ in this latter statement with $\nprev$. To do this it suffices to show that both $\KB \brev \alpha$ and $\KB \brev (\alpha \vee \beta)$ are consistent which, given the subset relation between the two, reduces to showing $\KB \brev (\alpha \vee \beta)$ is consistent. Assume, for contradiction, that $\KB \brev (\alpha \vee \beta)$ is inconsistent. Then we would have $\KB \nprev (\alpha \vee \beta) = \KB$ and so, from the initial assumption $\alpha \vee \beta \in \KB \nprev (\alpha \vee \beta)$, we get $\alpha \vee \beta \in \KB$. But from the following property for $\brev$: $\gamma \in \KB$ implies $\KB \brev \gamma$ is consistent, we would then get $\KB \brev (\alpha \vee \beta)$ is consistent - contradiction. Hence $\KB \brev (\alpha \vee \beta)$ is consistent, as required. 

For Disjunctive Overlap let us split into cases according to whether $\alpha \vee \beta \in \KB \nprev (\alpha \vee \beta)$ or not. Suppose first we do have $\alpha \vee \beta \in \KB \nprev (\alpha \vee \beta)$. If $\alpha \notin \KB \nprev \alpha$ or $\beta \notin \KB \nprev \beta$ then we obtain the desired $\KB \nprev \alpha \cap \KB \nprev \beta \subseteq \KB \nprev (\alpha \vee \beta)$ by the just proved P10. So assume $\alpha \in \KB \nprev \alpha$ and $\beta \in \KB \nprev \beta$. Then, by Lemma \ref{lemma:circ_equals_ast}, we have $\KB \nprev (\alpha \vee \beta) = \KB \brev (\alpha \vee \beta)$, $\KB \nprev \alpha = \KB \brev \alpha$ and $\KB \nprev \beta = \KB \brev \beta$ and so we can conclude using the fact that $\brev$ satisfies Disjunctive Overlap. Now lets turn to the case $\alpha \vee \beta \notin \KB \nprev (\alpha \vee \beta)$. Then we must have $\KB \brev (\alpha \vee \beta)$ is inconsistent, and $\KB \nprev (\alpha \vee \beta) = \KB$, which means we need to show $\KB \nprev \alpha \cap \KB \nprev \beta \subseteq \KB$. But from $\KB \brev (\alpha \vee \beta)$ is inconsistent, we get that at least one of $\KB \brev \alpha$, $\KB \brev \beta$ is also inconsistent using the fact that $\brev$ satisfies the following property:
\[
\KB \brev (\alpha \vee \beta) \subseteq \KB \brev \alpha\
\textrm{or }
\KB \brev (\alpha \vee \beta) \subseteq \KB \brev \beta.
\]
Hence at least one of $\KB \nprev \alpha, \KB \nprev \beta$ must also be equal to $\KB$ and so the conclusion follows. 

For P13 suppose $\alpha \in \KB \nprev \alpha$ and $\beta \in \KB \nprev \beta$. Then, by Lemma \ref{lemma:circ_equals_ast}, $\KB \nprev \alpha = \KB \brev \alpha$ and $\KB \nprev \beta = \KB \brev \beta$. Thus we need to show $\KB \nprev (\alpha \vee \beta) \subseteq \KB \brev \alpha \cup \KB \brev \beta$. If we can show $\KB \nprev (\alpha \vee \beta) = \KB \brev (\alpha \vee \beta)$ then this will follow by Disjunctive Rationality for $\brev$. But from $\alpha \in \KB \nprev \alpha$ and $\beta \in \KB \nprev \beta$ we obtain $\alpha \vee \beta \in \KB \nprev \alpha \cap \KB \nprev \beta$ by Closure for $\nprev$. Thus  $\alpha \vee \beta \in \KB \nprev (\alpha \vee \beta)$ by the just proved Disjunctive Overlap for $\nprev$, and so we obtain the desired $\KB \nprev (\alpha \vee \beta) = \KB \brev (\alpha \vee \beta)$ by Lemma \ref{lemma:circ_equals_ast}.

\noindent
{\em (Completeness)}
Let $\nprev$ be an operator that satisfies all the postulates listed in the statement of the theorem. By Theorem \ref{Thm:RepResultBOB} it suffices to show the existence of an operator $\brev$ that satisfies all the revision postulates from Theorem \ref{Thm:RepResultBOB} and is such that $\nprev$ is defined from $\brev$.
We define $\brev$ from $\nprev$ as follows:
\[
\KB \brev \alpha = 
\left\{
\begin{array}{ll}
\KB \nprev \alpha & \textrm{if } \alpha \in \KB \nprev \alpha \\
Cn(\bot) & \textrm{otherwise}
\end{array}
\right.
\]
Let's show $\brev$ so defined satisfies all the postulates. Tautology and Closure hold since $\nprev$ satisfies them. Success and Extensionality are obvious. 

For Disjunctive Overlap $\KB \brev \alpha \cap \KB \brev \beta \subseteq \KB \brev (\alpha \vee \beta)$ we split into three cases: 
\begin{itemize}
\item    
{\em (a)} $\alpha \notin \KB \nprev \alpha$ and $\beta \notin \KB \nprev \beta$. Then by P11 $\alpha \vee \beta \notin \KB \nprev (\alpha \vee \beta)$ and so we have $\KB \brev \alpha \cap \KB \brev \beta = Cn(\bot) = \KB \brev (\alpha \vee \beta)$, giving the required inclusion. 

\item
{\em (b)} $\alpha \in \KB \nprev \alpha$ and $\beta \notin \KB \nprev \beta$ (the case with $\alpha, \beta$ roles reversed is symmetrical). 
Then $\KB \brev \alpha = \KB \nprev \alpha$ and $\KB \brev \beta = Cn(\bot)$ and so we must show $\KB \nprev \alpha \subseteq \KB \brev (\alpha \vee \beta)$. If $\alpha \vee \beta \notin \KB \nprev (\alpha \vee \beta)$ then $\KB \brev (\alpha \vee \beta) = Cn(\bot)$ and we are done. If $\alpha \vee \beta \in \KB \nprev (\alpha \vee \beta)$ then $\KB \brev (\alpha \vee \beta) = \KB \nprev (\alpha \vee \beta)$ and then we get the required conclusion by P12.

\item
{\em (c)} $\alpha \in \KB \nprev \alpha$ and $\beta \in \KB \nprev \beta$. Then also $\alpha \vee \beta \in \KB \nprev (\alpha \vee \beta)$ by P14 and so we have $\KB \brev \alpha = \KB \nprev \alpha$, $\KB \brev \beta = \KB \nprev \beta$ and $\KB \brev (\alpha \vee \beta) = \KB \nprev (\alpha \vee \beta)$ and so we obtain Disjunctive Overlap for $\brev$ directly from Disjunctive Overlap for $\nprev$.
\end{itemize}

For Disjunctive Rationality, 
if $\alpha \notin \KB \nprev \alpha$ or $\beta \notin \KB \nprev \beta$ then at least one of $\KB \brev \alpha$ or $\KB \brev \beta$ is equal to $Cn(\bot)$ and so we are done. So assume $\alpha \in \KB \nprev \alpha$ and $\beta \in \KB \nprev \beta$, so also, by P14, $\alpha \vee \beta \in \KB \nprev (\alpha \vee \beta)$. From these three statements we know $\KB \brev \alpha = \KB \nprev \alpha$, $\KB \brev \beta = \KB \nprev \beta$ and $\KB \brev (\alpha \vee \beta) = \KB \nprev (\alpha \vee \beta)$, so it remains to show $\KB \nprev (\alpha \vee \beta) \subseteq \KB \nprev \alpha \cup \KB \nprev \beta$. But this holds by P13 for $\nprev$.

For Inclusion it suffices, given all the other postulates, to show $\KB \brev \top = \KB$. By construction we know $\KB \brev \top = \KB \nprev \top$. But $\KB \nprev \top = \KB$, as required, using the fact that $\nprev$ satisfies Inclusion.

Having shown that $\brev$ satisfies all the postulates, it remains to show 
\[
\KB \nprev \alpha = 
\left\{
\begin{array}{ll}
     \KB \brev \alpha &  \textrm{if this is consistent}  \\
     \KB & \textrm{otherwise} 
\end{array}
\right.
\]
Suppose first $\KB \brev \alpha$ is consistent. By construction $\KB \brev \alpha$ equals either $\KB \nprev \alpha$ or $Cn(\bot)$, so in this case we must have $\KB \nprev \alpha = \KB \brev \alpha$ as required. Now suppose $\KB \brev \alpha$ is inconsistent. Then since $\nprev$ satisfies Consistency we must have $\KB \nprev \alpha \neq \KB \brev \alpha$ and so, by construction of $\brev$, $\alpha \notin \KB \nprev \alpha$. Then, since $\nprev$ satisfies Relative Success, $\KB \nprev \alpha = \KB$, as required.
\end{proof}

\ztransbiordercredrevision*
\begin{proof}
{\em (Soundness)}
Suppose $\nprev$ is a ZTBOB NPR operator which associates a $z$-transitive biorder $\biord_\KB$ to each $\KB$.
By Theorem \ref{theorem:biorder_cred_revision} we only need to show $\nprev$ satisfies the extra part of P12+. Let $\BM = \langle \lo, \up \rangle$ be a compressed biorder-based interpretation for $\biord_\KB$. By Proposition \ref{Prop:zTransitivityConditionOnBOB} we know $\BM$ satisfies BZT.
Suppose $\alpha \vee \beta \in \KB \nprev (\alpha \vee \beta)$ and $\beta \notin \KB \nprev \beta$.  From the former there is~$u \in \Mod{\alpha \vee \beta}$ \suchthat\ $\lo(u) \leq \up(\alpha \vee \beta)$. Note this gives also $\lo(u) \leq \up(\beta)$ and so, since $\beta \notin \KB \nprev \beta$ we must have $u \notin \Mod{\beta}$, i.e., $u \in \Mod{\alpha}$. Meanwhile $\beta \notin \KB \nprev \beta$ also implies there is~$z \in \Mod{\beta}$ \suchthat\ $\up(z) < \lo(\beta)$, and in particular $\up(z) < \lo(z)$. Since $\lo(u) \leq \up(\beta)$ we have $\lo(u) \leq \up(z)$. We must show $\KB \nprev (\alpha \vee \beta) \subseteq \KB \nprev \alpha$, equivalently $\Mod{\KB \nprev \alpha} \subseteq \Mod{\KB \nprev (\alpha \vee \beta)}$. It suffices to show $\up(\alpha) \leq \up(\alpha \vee \beta)$. But if $\up(\alpha \vee \beta) < \up(\alpha)$ then we would have $\up(\alpha \vee \beta) = \up(\beta) = \up(z)$ and so, since $u \in \Mod{\alpha}$, $\up(z) < \up(u)$. But this contradicts condition BZT. Hence $\up(\alpha) \leq \up(\alpha \vee \beta)$ as required.

\noindent
{\em (Completeness)}
Suppose $\nprev$ satisfies all the postulates from Theorem \ref{theorem:biorder_cred_revision} with P12 replaced by P12+ and suppose $\brev$ is its corresponding BOB revision operator. By Theorem \ref{theorem:biorder_cred_revision} we know that, for each $\KB$, $\KB \nprev \alpha$ is determined from $\biord_{\BM}$ where $\BM = \BM^\brev_\KB$ is the canonical biorder-based interpretation associated to $\brev$ and $\KB$.
We show that $\BM$ additionally satisfies BZT, and hence that $\nprev$ is a ZTBOB NPR operator. Suppose for contradiction BZT doesn't hold. Then for some $u, z \in \U$ we have $\lo(u) \leq \up(z) < \up(u)$ and $\up(z) < \lo(z)$. We know there is~$w \in \U$ \suchthat\ $\lo(w) = \up(u)$ and $\lo(w) \leq \up(w)$. Now let $\alpha, \beta$ be \suchthat\ $\Mod{\alpha} = \{u, w\}$ and $\Mod{\beta}= \{z\}$. Then $\Mod{\KB \nprev (\alpha \vee \beta} = \{u\}$, $\Mod{\KB \nprev \beta} = \Mod{\KB}$ and $\Mod{\KB \nprev \alpha} = \{u,w\}$. Hence $\alpha \vee \beta \in \KB \nprev (\alpha \vee \beta)$, $\beta \notin \KB \nprev \beta$ and $\KB \nprev (\alpha \vee \beta) \nsubseteq \KB \nprev \alpha$ - contradicting P12+. Hence BZT must hold for $\BM$ as required.
\end{proof}

\transbiordercredrevision*

\begin{proof}
{\em (Soundness)}
Let $\nprev$ be a TBOB NPR operator. We only need to show Disjunctive Monotony holds. Let $\brev$ be a TBOB revision operator \suchthat\ $\nprev$ is defined from $\brev$ and suppose $\neg\beta \notin \KB \nprev \alpha$. We split into two cases.
\begin{itemize}
\item 
{\em (a)} $\KB \brev \alpha$ is inconsistent. 
Then $\KB \nprev \alpha  = \KB$ so $\neg\beta \notin \KB$. Then also $\neg(\alpha \vee \beta) \notin \KB$ and so (since TBOB revision satisfies both Inclusion and Vacuity) we have both $\KB \brev \beta = Cn(\KB \cup \{\beta\})$ and $\KB \brev (\alpha \vee \beta) = Cn(\KB \cup \{\alpha \vee \beta\})$ with both $\KB \brev \beta$ and $\KB \brev (\alpha \vee \beta)$ being consistent. Hence we get $\KB \nprev (\alpha \vee \beta) = \KB \brev (\alpha \vee \beta) \subseteq \KB \brev \beta = \KB \nprev \beta$ as required.

\item 
{\em (b)} $\KB \brev \alpha$ is consistent.
Then $\KB \nprev \alpha = \KB \brev \alpha$ and so $\neg\beta \notin \KB \brev \alpha$. Hence, by Disjunctive Monotony for $\brev$, $\KB \brev (\alpha \vee \beta) \subseteq \KB \brev \beta$. If $\KB \brev (\alpha \vee \beta)$ is inconsistent then so must be $\KB \brev \beta$ and then we would get the required conclusion via $\KB \nprev (\alpha \vee \beta) = \KB = \KB \nprev \beta$. So assume $\KB \brev (\alpha \vee \beta)$ is consistent, giving $'KB \nprev (\alpha \vee \beta) = \KB \brev (\alpha \vee \beta)$. Then we claim also $\KB \brev \beta$ is consistent, which would then give $\KB \nprev \beta = \KB \brev \beta$ and would allow us to conclude using Disjunctive Monotony for $\brev$. Suppose, for contradiction $\KB \brev \beta$ is inconsistent. Then $\KB \brev ((\alpha \vee \beta) \wedge \beta)$ is inconsistent by Extensionality for $\brev$ and so, since $\KB \brev ((\alpha \vee \beta) \wedge \beta) \subseteq Cn(\KB \brev (\alpha \vee \beta) \cup \{\beta\})$ by Superexpansion\footnote{$\KB \brev (\gamma \wedge \delta) \subseteq Cn(\KB \wedge \gamma \cup \{\delta\})$.}, which is sound for TBOB revision, we get $\neg \beta \in \KB \brev (\alpha \vee \beta)$. Therefore, by consistency of $\KB \brev (\alpha \vee \beta)$ and Success for $\brev$, $\neg\alpha \notin \KB \brev (\alpha \vee \beta)$. Hence, using Superexpansion and Extensionality for $\brev$ we get $Cn(\KB \brev (\alpha \vee \beta) \cup \{\alpha\}) \subseteq \KB \brev ((\alpha \vee \beta) \wedge \alpha) = \KB \brev \alpha$ and so, since $\neg\beta \in \KB \brev (\alpha \vee \beta)$ we get $\neg\beta \in \KB \brev \alpha$ - contradiction. 
\end{itemize}

\noindent
{\em (Completeness)}
Let $\nprev$ satisfy all the postulates from Theorem \ref{theorem:z_trans_biorder_cred_revision} plus Disjunctive Monotony. Let $\brev$ be defined from $\nprev$ as in the completeness proof of Theorem \ref{theorem:biorder_cred_revision}. It suffices to show $\brev$ satisfies Disjunctive Monotony (which for TBOB revision is equivalent to Subexpansion). So suppose $\neg\beta\notin \KB \brev \alpha$. We must show $\KB \brev (\alpha \vee \beta) \subseteq \KB \brev \alpha$. Since $\KB \brev \alpha$ is consistent we must have $\KB \brev \alpha = \KB \nprev \alpha$ and $\alpha \in \KB \nprev \alpha$. From the former we know $\neg\beta \notin \KB \nprev \alpha$. If $\KB \brev \beta$ is inconsistent then we are done, so assume $\KB \brev \beta$ is consistent. Then we must have $\KB \brev \beta = \KB \nprev \beta$ and $\beta \in \KB \nprev \beta$. From the latter and $\alpha \in \KB \nprev \alpha$ we get $\alpha \vee \beta \in \KB \nprev (\alpha \vee \beta)$ by P14 for $\nprev$. Hence $\KB \brev (\alpha \vee \beta) = \KB \nprev (\alpha \vee \beta)$ and we get the desired conclusion from Disjunctive Monotony for $\nprev$.
\end{proof}

Before proving Theorem \ref{theorem:ZTBOBasCL} we need the following lemma.

\begin{lemma}
\label{lemma:consequenceOfConditionOnCredsetAndIOB}
Let $\tuple{\ast,\CredSet}$ be a CL structure \suchthat\ $\ast$ is an IOB revision operator. Then $\CredSet$ and $\ast$ together satisfy:
\[
\textrm{If }
\alpha \vee \beta \in \CredSet_\KB\
\textrm{and }
\beta \notin \CredSet_\KB\
\textrm{then }
\KB \ast (\alpha \vee \beta)
=
\KB \ast \alpha
\]
\end{lemma}

\begin{proof}
Suppose $\alpha \vee \beta \in \CredSet_\KB$ and $\beta \notin \CredSet_\KB$. From the latter we get $(\alpha \vee \beta) \wedge \beta \notin \CredSet_\KB$ by $\CredSet 2$ and then from this and $\alpha \vee \beta \in \CredSet_\KB$ we get $\neg\beta \in \KB \ast (\alpha \vee \beta)$ by condition (\ref{condition_on_star+credset}). By reasoning about an arbitrary interval-based interpretation, it can be shown that any IOB revision operator satisfying this must have $\KB \ast (\alpha \vee \beta) = \KB \ast \alpha$ as required. 
\end{proof}

\ZTBOBasCL*

\begin{proof}
{\em (`If')}
We go through the list of the characteristic postulates of ZTBOB NPR operators mentioned after Theorem \ref{theorem:z_trans_biorder_cred_revision}. Closure, Relative Success, Inclusion, Extensionality and Consistency all straightforwardly follow using the fact that IOB revision operators satisfy the corresponding postulates. 

For Disjunctive Overlap, we consider two cases
\begin{itemize}
\item
{\em (a)} $\alpha \vee \beta \notin \CredSet_\KB$. Then $\KB \nprev_{\ast,\CredSet} (\alpha \vee \beta) =\KB$. From $\CredSet 4'$ we have $\alpha \notin \CredSet_\KB$ or $\beta \notin \CredSet_\KB$. Hence $\KB \nprev_{\ast,\CredSet} \alpha = \KB$ or $\KB \nprev_{\ast,\CredSet} \beta = \KB$. In either case we obtain the desired $\KB \nprev_{\ast,\CredSet} \alpha \cap \KB \nprev_{\ast,\CredSet} \beta \subseteq \KB \nprev_{\ast,\CredSet} (\alpha \vee \beta)$. 

\item
{\em (b)}
$\alpha \vee \beta \in \CredSet_\KB$. Then $\KB \nprev_{\ast,\CredSet} (\alpha \vee \beta) =\KB \ast (\alpha \vee \beta)$. By $\CredSet 3$ at least one of $\alpha, \beta$ is in $\CredSet_\KB$. If they are both in $\CredSet_\KB$ then $\KB \nprev_{\ast,\CredSet} \alpha = \KB \ast \alpha$ and $\KB \nprev_{\ast,\CredSet} \beta = \KB \ast \beta$ and we get the desired conclusion by Disjunctive Overlap for $\ast$. If only one of them is in $\CredSet_\KB$, say $\alpha$ (the case for $\beta$ is similar), then $\KB \nprev_{\ast,\CredSet} \alpha \cap \KB \nprev_{\ast,\CredSet} \beta = \KB \ast \alpha \cap \KB \subseteq \KB \ast \alpha$. But $\KB \ast \alpha = \KB \ast (\alpha \vee \beta)$ by Lemma \ref{lemma:consequenceOfConditionOnCredsetAndIOB}, giving the required conclusion.   
\end{itemize}

For Disjunctive Rationality, we again consider two cases.
\begin{itemize}
\item 
{\em (a)}
$\alpha \vee \beta \notin \CredSet_\KB$. Then $\KB \nprev_{\ast,\CredSet} (\alpha \vee \beta) = \KB$. By $\CredSet 4'$ at least one of $\alpha$ and $\beta$ is not in $\CredSet_\KB$ so at least one of $\KB \nprev_{\ast,\CredSet} \alpha$ and $\KB \nprev_{\ast,\CredSet} \beta$ is equal to $\KB$, from which the conclusion follows. 

\item
{\em (b)}
$\alpha \vee \beta \in \CredSet_\KB$. Then $\KB \nprev_{\ast,\CredSet} (\alpha \vee \beta) = \KB \ast (\alpha \vee \beta)$. By $\CredSet 3$ at least one of $\alpha$ and $\beta$ is in $\CredSet_\KB$. If they are both in $\CredSet_\KB$ then $\KB \nprev_{\ast,\CredSet} \alpha = \KB \ast \alpha$ and $\KB \nprev_{\ast,\CredSet} \beta = \KB \ast \beta$ and we conclude using the fact that $\ast$ satisfies Disjunctive Rationality. If only one of them is in $\CredSet_\KB$, say $\alpha$ (the case for $\beta$ is similar) then $\KB \nprev_{\ast,\CredSet} \alpha \cup \KB \nprev_{\ast,\CredSet} \beta = \KB \ast \alpha \cup \KB \supseteq \KB \ast \alpha$. But $\KB \ast \alpha = \KB \ast (\alpha \vee \beta)$ by Lemma \ref{lemma:consequenceOfConditionOnCredsetAndIOB}, giving the required conclusion. 
\end{itemize}

For P11 suppose $\alpha \vee \beta \in \KB \nprev_{\ast,\CredSet} (\alpha \vee \beta)$. If $\alpha \vee \beta \notin \CredSet_\KB$ then we would have $\alpha \vee \beta \notin \KB$ by $\CredSet 1$. But then we must be in the first clause of the definition of $\nprev_{\ast, \CredSet}$, i.e., $\KB \nprev_{\ast,\CredSet} (\alpha \vee \beta) = \KB \ast (\alpha \vee \beta)$ and $\alpha \vee \beta \in \CredSet_\KB$ - contradiction.  Hence $\alpha \vee \beta \in \CredSet_\KB$ and then we obtain the desired conclusion using $\CredSet 3$ and Success for $\ast$. 

For P12+ Suppose $\alpha \vee \beta \in \KB \nprev_{\ast,\CredSet} (\alpha \vee \beta)$ and $\beta \notin \KB \nprev_{\ast,\CredSet} \beta$. From the former we obtain, as in the proof of P11, that $\alpha \vee \beta \in \CredSet_\KB$ while the latter gives us $\beta \notin \CredSet_\KB$. From $\CredSet 3$ we get $\alpha \in \CredSet_\KB$. So $\KB \nprev_{\ast,\CredSet} \alpha = \KB \ast \alpha$ and $\KB \nprev_{\ast,\CredSet} (\alpha \vee \beta) = \KB \ast (\alpha \vee \beta)$. But we know $\KB \ast \alpha = \KB \ast (\alpha \vee \beta)$ from Lemma \ref{lemma:consequenceOfConditionOnCredsetAndIOB} as required.

\noindent
{\em (`Only if')}
Let $\nprev$ be a ZTBOB NPR operator. From $\nprev$ we define $\ast^\nprev$ and $\CredSet^\nprev$ as in the main text following Theorem \ref{theorem:ZTBOBasCL}. We need to show that {\em (i)} $\tuple{\ast^\nprev, \CredSet^\nprev}$ forms a CL structure, {\em (ii)} $\ast^\nprev$ is an IOB revision operator, and {\em (iii)} $\nprev = \nprev_{\ast', \CredSet'}$ where $\ast' = \ast^\nprev$ and $\CredSet' = \CredSet^\nprev$. 

First we show $\CredSet$ satisfies $\CredSet 1$-$\CredSet 3$ and $\CredSet 4'$. For $\CredSet 1$ suppose $\alpha \in \KB$. Suppose for contradiction $\alpha \notin \KB \nprev \alpha$. Then by Relative Success we would have $\KB \nprev \alpha = \KB$ and so $\alpha \notin \KB$ - contradiction. Hence $\alpha \in \KB \nprev \alpha$ as required. $\CredSet 2$ is straightforward. $\CredSet 3$ and $\CredSet 4'$ follow directly from P11 and P14 respectively.

Next we show $\ast^\nprev$ is an IOB revision operator, for which it is sufficient to show the relation $\intord_\KB$ defined in the main text after Theorem \ref{theorem:ZTBOBasCL} is an interval order, i.e., satisfies Reflexivity and Ferrers condition. Reflexivity is straightforward, so let's show Ferrers. Suppose $u \intord_\KB v$ and $u' \intord_\KB v'$. We must show $u \intord_\KB v'$ or $u' \intord_\KB v$. If at least one of $v, v'$ is in $\mathit{Diss}(\biord_\KB)$ then we are done, so suppose neither is. Then, by definition of $\intord_\KB$, our assumptions give us also that neither $u$ nor $u'$ is in $\mathit{Diss}(\biord_\KB)$, and $u \biord_\KB v$, $u' \biord_\KB  v'$. Since $\biord_\KB$ satisfies Ferrers we must have $u \biord_\KB v'$ or $u' \biord_\KB v$. Applying the definition of $\intord_\KB$, the former case gives us $u \intord_\KB v'$ while the latter gives us $u' \intord_\KB v$ as required.

Before showing that $\CredSet^\nprev$ and $\ast^\nprev$ satisfy the joint condition (\ref{condition_on_star+credset}) we need the following lemma, which shows that, when transforming a $z$-transitive biorder $\biord_\KB$ into the interval order $\intord_\KB$, the optimal elements of credible sentences are preserved.

\begin{lemma}
\label{lemma:nonemptiespreserved}
Let $\biord_\KB$ be a $z$-transitive biorder and let $\intord_\KB$ be defined from $\biord_\KB$ as in the main text after Theorem \ref{theorem:ZTBOBasCL}. For any $\Val \subseteq \U$ \suchthat\ $\opt(\Val , \biord_\KB) \neq \emptyset$ we have $\opt(\Val , \intord_\KB) = \opt(\Val , \biord_\KB)$.
\end{lemma}

\begin{proof}
Suppose $\opt(\Val , \biord_\KB) \neq \emptyset$. For the `$\supseteq$' inclusion let $u \in \opt(\Val , \biord_\KB)$ and let $v \in \Val$. We must show $u \intord_\KB v$. If $v \in \mathit{Diss}(\biord_\KB)$ we are done. So assume $v \notin \mathit{Diss}(\biord_\KB)$. From $u \in \opt(\Val , \biord_\KB)$ we know both $u \notin \mathit{Diss}(\biord_\KB)$ and $u \biord_\KB v$, which gives $u \intord_\KB v$ as required. (Note we did not need the assumption $\opt(\Val , \biord_\KB) \neq \emptyset$, or that $\biord_\KB$ is $z$-transitive for this part.) 

For the `$\subseteq$' inclusion let $u \in \opt(\Val , \intord_\KB)$ and let $v \in \Val$. We must show $u \biord_\KB v$. By the optimality of $u$ we know $u \intord_\KB v$. If $v \notin \mathit{Diss}(\biord_\KB)$ then we get the required conclusion by definition of $\intord_\KB$. So assume $v \in \mathit{Diss}(\biord_\KB)$, i.e., $v \not\biord_\KB v$. By the assumption $\opt(\Val , \biord_\KB) \neq \emptyset$ there is some~$z \in \opt(\Val , \biord_\KB)$. So we have $u \intord_\KB z$, $z \biord_\KB u$, $z \biord_\KB v$ and $z \biord_\KB z$. From $u \intord_\KB z$ and $z \notin \mathit{Diss}(\biord_\KB)$ we have $u \biord_\KB z$ and then applying $z$-transitivity to this together with $z \biord_\KB v$ and $v \not\biord_\KB v$ gives $u \biord_\KB v$ as required. 
\end{proof}

We now continue with the `only if' part of the proof of Theorem \ref{theorem:ZTBOBasCL} by showing $\CredSet^\nprev$ and $\ast^\nprev$ satisfy the condition (\ref{condition_on_star+credset}). Suppose $\alpha \in \CredSet^\nprev_\KB$ and $\alpha \wedge \beta \notin \CredSet^\nprev_\KB$, i.e., $\alpha \in \KB \nprev \alpha$ and $\alpha \wedge \beta \notin \KB \circ (\alpha \wedge \beta)$. We must show $\neg\beta \in \KB \ast^\nprev \alpha$, i.e., $\opt(\Mod{\alpha}, \intord_\KB) \subseteq \Mod{\neg\beta}$. Since $\alpha \in \KB \nprev \alpha$ we know $\opt(\Mod{\alpha}, \biord_\KB) \neq \emptyset$ and so, by Lemma \ref{lemma:nonemptiespreserved} it suffices to show $\opt(\Mod{\alpha}, \biord_\KB) \subseteq \Mod{\neg\beta}$, i.e., $\neg\beta \in \KB \nprev \alpha$. But the following rule can be shown to hold for all BOB NPR operators $\nprev'$:
\[
\textrm{If }
\alpha \in \KB \nprev' \alpha\
\textrm{and }
\alpha \wedge \beta \notin \KB \nprev' (\alpha \wedge \beta)\
\textrm{then }
\neg\beta \in \KB \nprev' \alpha
\]
Then from the above rule with $\alpha \in \KB \nprev \alpha$ and $\alpha \wedge \beta \notin \KB \nprev (\alpha \wedge \beta)$ we obtain $\neg\beta \in \KB \nprev \alpha$ as required.

Having established that $\tuple{\ast^\nprev,\CredSet^\nprev}$ forms a CL structure and that $\ast^\nprev$ is an IOB revision operator, all that remains is to show that, for all $\KB$, $\alpha$ we have $\KB \nprev_{\ast,\CredSet} \alpha = \KB \nprev \alpha$. We split into two cases.
\begin{itemize}
\item 
{\em (a)} $\alpha \in \CredSet_\KB$. Then $\KB \nprev_{\ast,\CredSet} \alpha = \KB \ast^\nprev \alpha$. By definition of $\CredSet^\nprev$ we have $\alpha \in \KB \nprev \alpha$ and so $\opt(\Mod{\alpha}, \biord_\KB) \neq \emptyset$. But then, by Lemma \ref{lemma:nonemptiespreserved}, $\opt(\Mod{\alpha}, \intord_\KB) = \opt(\Mod{\alpha}, \biord_\KB)$, i.e., $\KB \ast^\nprev \alpha = \KB \nprev \alpha$ as required.

\item 
{\em (b)} $\alpha \notin \CredSet^\nprev$. Then $\KB \nprev_{\ast,\CredSet} \alpha = \KB$. By definition of $\CredSet^\nprev$ we have $\alpha \notin \KB \nprev \alpha$ and so also $\KB \nprev \alpha = \KB$ by Relative Success for $\nprev$.
\end{itemize}
This completes the proof of Theorem \ref{theorem:ZTBOBasCL}.
\end{proof}

\TBOBasCL*

\begin{proof}
{\em (`If')}
Let $\tuple{\ast,\CredSet}$ be \suchthat\ $\ast$ is an AGM revision operator. Since every AGM revision operator is an IOB revision operator, we already know by Theorem \ref{theorem:ZTBOBasCL} that $\nprev_{\ast,\CredSet}$ is a ZTBOB NPR operator. To show it is also a TBOB NPR operator it suffices, by Theorem \ref{theorem:trans_biorder_cred_revision}, to show $\nprev_{\ast,\CredSet}$ satisfies Disjunctive Monotony, i.e.,
\begin{equation*}
\text{If }\lnot\beta\notin\KB\nprev\alpha, \text{ then }\KB\nprev(\alpha\lor\beta)\subseteq\KB\nprev\beta. 
\end{equation*}
So suppose $\neg\beta \notin \KB \nprev_{\ast,\CredSet} \alpha$. We split into four separate cases according to whether $\alpha, \beta$ are credible or not.
\begin{itemize}
\item 
{\em (a)}
$\alpha, \beta \in \CredSet_\KB$. 
Then $\alpha \vee \beta \in \CredSet_\KB$ by $\CredSet 4'$ and so we obtain Disjunctive Monotony for $\nprev_{\ast,\CredSet}$ using the fact that AGM revision $\ast$ satisfies it.

\item 
{\em (b)}
$\{\alpha, \beta\} \cap \CredSet_\KB = \emptyset$. Then $\alpha \vee \beta \notin \CredSet_\KB$ by $\CredSet 3$ and so $\KB \nprev_{\ast,\CredSet} (\alpha \vee \beta) = \KB = \KB \nprev_{\ast,\CredSet} \beta$.

\item 
{\em (c)}
$\alpha \in \CredSet_\KB$ and $\beta \notin \CredSet_\KB$. In this case we claim $\alpha \vee \beta \notin \CredSet_\KB$, which will then allow us to conclude with $\KB \nprev_{\ast,\CredSet} (\alpha \vee \beta) = \KB = \KB \nprev_{\ast,\CredSet} \beta$ as in the preceding case. For suppose, for contradiction, $\alpha \vee \beta \in \CredSet_\KB$. From $\beta \notin \CredSet_\KB$ we get $(\alpha \vee \beta) \wedge \beta \notin \CredSet_\KB$ by $\CredSet 2$. Then, by condition ({\ref{condition_on_star+credset}}), $\neg\beta \in \KB \ast (\alpha \vee \beta)$. Since $\ast$ is AGM this then gives us $\neg\beta \in \KB \ast \alpha$. But, since $\alpha \in \CredSet_\KB$, this leads to a contradiction with the initial assumption $\neg\beta \notin \KB \nprev_{\ast,\CredSet} \alpha$.

\item 
{\em (d)}
$\alpha \notin \CredSet_\KB$ and $\beta \in \CredSet_\KB$. Then the assumption $\neg\beta \notin \KB \nprev_{\ast,\CredSet} \alpha$ becomes $\neg\beta \notin \KB$. If $\alpha \vee \beta \notin \CredSet_\KB$ then the consequent of Disjunctive Monotony reduces to $\KB \subseteq \KB \ast \beta$, which follows from $\neg\beta \notin \KB$ using the properties of AGM revision. If $\alpha \vee \beta \in \CredSet_\KB$ then the consequent of Disjunctive Monotony reduces to $\KB \ast (\alpha \vee \beta) \subseteq \KB \ast \beta$, which again follows from $\neg\beta \notin \KB$ using the properties of AGM revision. 
\end{itemize}

\noindent
{\em ('Only if')}
Let $\nprev$ be a TBOB NPR operator and let $\ast^\nprev$ and $\CredSet^\nprev$ be defined from $\nprev$ as in the main text following Theorem \ref{theorem:ZTBOBasCL}. Since $\nprev$ is also a ZTBOB NPR operator, we already know from the proof of that theorem that $\nprev = \nprev_{\ast',\CredSet'}$ where $\ast' = \ast^\nprev$ and $\CredSet' = \CredSet^\nprev$, and that $\intord_\KB$ is an interval order, where $\intord_\KB$ is defined from the biorder $\biord_\KB$ that $\nprev$ associates to $\KB$. It suffices to show that the extra assumption that $\biord_\KB$ is fully transitive is enough to show that $\intord_\KB$ is not just an interval order but a total preorder, and therefore that $\ast^\nprev$ is an AGM revision operator. We need to show that $\intord_\KB$ is complete and transitive. Completeness holds since every interval order is complete (which can easily be shown from Reflexivity and Ferrers). To show transitivity suppose $u \intord_\KB v$ and $v \intord_\KB w$. We must show $u \intord_\KB w$. If $w \in \mathit{Diss}(\biord_\KB)$ then we get the required conclusion by definition of $\intord_\KB$. So assume $w \notin \mathit{Diss}(\biord_\KB)$. Then, since $v \intord_\KB w$ we must have $v \notin \mathit{Diss}(\biord_\KB)$ and $v \biord_\KB w$. From $v \notin \mathit{Diss}(\biord_\KB)$ and $u \intord_\KB v$ we also have $u \notin \mathit{Diss}(\biord_\KB)$ and $u \biord_\KB v$. Hence from $u \biord_\KB v$ and $v \biord_\KB w$ and the assumption $\biord_\KB$ is transitive we get $u \biord_\KB w$. From this and $u,v \notin \mathit{Diss}(\biord_\KB)$ we get the required $u \intord_\KB w$.
\end{proof}

\end{document}